\documentclass{ecai}
\usepackage{times}
\usepackage{graphicx}
\usepackage{latexsym}

\usepackage{times}
\usepackage{helvet}
\usepackage{courier}
\usepackage{amsfonts}
\usepackage{amsmath}
\usepackage{amssymb}
\usepackage{algorithm}
\usepackage{algorithmic}
\usepackage{booktabs}
\usepackage{color}
\usepackage{chngcntr}
\usepackage{comment}
\usepackage{enumitem}
\usepackage{hhline}
\usepackage{IEEEtrantools}
\usepackage[utf8]{inputenc}
\usepackage{mathtools}
\usepackage{multirow}
\usepackage{setspace}
\usepackage{siunitx}
\usepackage{subcaption}
\usepackage{tikz}
  \usetikzlibrary{calc}
  \usetikzlibrary{trees}
\usepackage{tikz-3dplot}
\usepackage{tikz-qtree}
\usepackage{url}
\usepackage{verbatim}
\usepackage{wrapfig}
\usepackage{calc}
\usepackage{etoolbox}

\providecommand{\citet}[1]{\citeauthor{#1}~(\citeyear{#1})}
\definecolor{blue}{rgb}{0,0,1}

\makeatletter
\newcommand*{\inlineequation}[2][]{%
  \begingroup
    \refstepcounter{equation}%
    \ifx\\#1\\%
    \else
      \label{#1}%
    \fi
    \relpenalty=10000 %
    \binoppenalty=10000 %
    \ensuremath{%
      #2%
    }%
    ~\@eqnnum
  \endgroup
}
\makeatother

\newcommand{\xqedhere}[2]{%
  \rlap{\hbox to#1{\hfil\llap{\ensuremath{#2}}}}}

\DeclareMathOperator*{\argmin}{argmin}
\DeclareMathOperator*{\argmax}{argmax}

\usepackage{accents}

\newcommand*{\vv}[1]{\vec{\mkern0mu#1}}

\newcommand{\bgmodel}{\mathcal{B}}
\newcommand{\posgmodel}{\mathcal{P}}

\newcommand{\numeq}[1]{\mathrel{\overset{\{#1\}}{=}}}

\newcommand{\format}[3]{{{#1}_{#2}^{#3}}}

\newcommand{\stat}[2][]{\sigma_{#1}^{#2}}

\newcommand{\statcond}[2][]{\format{\sigma}{c,#1}{#2}}
\newcommand{\statmarg}[2][]{\format{\sigma}{m,#1}{#2}}
\newcommand{\vstatmarg}[2][]{\format{\vec{\sigma}}{m,#1}{#2}}

\newcommand{\statspace}[2][]{\Delta( \vv{\Theta}\vphantom{\Theta}_{{#1}}^{{#2}}  )}

\newcommand{\AOH}[2][]{\vv{\theta}\format{\vphantom{\delta}}{#1}{#2}}

\newcommand{\jointAOH}[1]{ \langle \AOH[1]{#1}, \AOH[2]{#1} \rangle }
\newcommand{\jointtype}[1]{ \langle \theta_1^{#1}, \theta_2^{#1} \rangle }
\newcommand{\AOHs}[2][]{\vv{\Theta}\format{\vphantom{\delta}}{#1}{#2}}

\newcommand{\st}[2][]{\format{s}{#1}{#2}}
\newcommand{\act}[2][]{\format{a}{#1}{#2}}
\newcommand{\jointact}[1]{ \langle \act[1]{#1}, \act[2]{#1} \rangle }
\newcommand{\obs}[2][]{\format{o}{#1}{#2}}
\newcommand{\jointobs}[1]{\obs[1]{#1} \obs[2]{#1}}

\newcommand{\dr}[2][]{\format{\delta}{#1}{#2}}
\newcommand{\jointdr}[1]{ \langle \dr[1]{#1}, \dr[2]{#1} \rangle}

\newcommand{\ppol}[2][]{\format{\pi}{#1}{#2}}

\newcommand{\pjp}[2][]{\format{\varphi}{#1}{#2}}

\newcommand{\jointppol}[1]{ \langle \ppol[1]{#1}, \ppol[2]{#1} \rangle }

\newcommand{\sdas}[1]{\Delta_{#1}^S}  
\newcommand{\spas}[2][]{\format{\Pi}{#1}{#2}}   

\newcommand{\V}[2][]{\format{V}{#2}{#1}}
\newcommand{\Q}[2][]{\format{Q}{#2}{#1}}

\newcommand{\optformat}[3]{{{#1}^{#2*}_{#3}}}
\newcommand{\optdr}[2][]{\optformat{\delta}{#2}{#1}}
\newcommand{\optV}[2][]{\optformat{V}{#1}{#2}}
\newcommand{\optQ}[2][]{\optformat{Q}{#1}{#2}}

\usepackage{amsthm}
\newtheorem{customtheorem}{Theorem}
\newtheorem{lemma}{Lemma}
\newtheorem{corollary}{Corollary}
\theoremstyle{definition}

\newtheorem{definition}{Definition}

\makeatletter
\g@addto@macro\normalsize{
  \setlength\abovedisplayskip{4pt}
  \setlength\belowdisplayskip{4pt}
  \setlength\abovedisplayshortskip{2pt}
  \setlength\belowdisplayshortskip{2pt}
}
\makeatother

\title{Structure in the Value Function of Two-Player\\
       Zero-Sum Games of Incomplete Information}
\author{
    Auke J. Wiggers\institute{
        Scyfer B.V.,
	    email: auke@scyfer.nl
    }
\and
    Frans A. Oliehoek\institute{
       University of Liverpool,
       Universiteit van Amsterdam,
       email: Frans.Oliehoek@liverpool.ac.uk
} 
\and
    Diederik M. Roijers\institute{
        University of Oxford,
        email: diederik.roijers@cs.ox.ac.uk
    } 
}
\begin{document}

\maketitle

\begin{abstract}
Zero-sum stochastic games provide a rich model for competitive decision making.
However, under general forms of state uncertainty as considered in the Partially
Observable Stochastic Game (POSG), such decision making problems are still not
very well understood. This paper makes a contribution to the theory of zero-sum
POSGs by characterizing structure in their value function. In particular, we
introduce a new formulation of the value function for zs-POSGs as a function of
the `plan-time sufficient statistics' (roughly speaking the information
distribution in the POSG), which has the potential to enable generalization over
such information distributions. We further delineate this generalization
capability by proving a structural result on the shape of value function: it
exhibits concavity and convexity with respect to appropriately chosen marginals
of the statistic space.  This result is a key pre-cursor for developing solution
methods that may be able to exploit such structure.  Finally, we show how these
results allow us to reduce a zs-POSG to a ‘centralized’ model with shared
observations, thereby transferring results for the latter, narrower class, to
games with individual (private) observations.
\end{abstract}

\section{Introduction}
Modeling decision making for strictly competitive settings with incomplete information is a field with many promising applications for AI. 
Examples include games such as poker \cite{rubin2011computer} and security settings \cite{jain2011double}. 
In strictly competitive sequential games in which the environment can be influenced through actions, the problem of behaving rationally can be modeled as a \emph{zero-sum Partially Observable Stochastic Game (zs-POSG)}.

Reasoning about zs-POSGs poses a challenge for strategic agents:
they need to reason about their own uncertainty regarding the state of the environment as well as uncertainty regarding the opposing agent. 
As this opponent is trying to minimize the reward that they are maximizing, behaving strategically typically requires stochastic strategies. 
A factor that further complicates the reasoning is that agents not only influence their immediate rewards, but also the \emph{future state of the environment} and \emph{both agents' future observations}. 

In this paper, we prove the existence of structural properties of zs-POSGs, that may be exploited to make reasoning about these models more tractable. 
We take inspiration from recent work for collaborative settings which has shown that it is possible to summarize the past joint policy using so called plan-time sufficient statistics \cite{Oliehoek13IJCAI}, which can be interpreted as the belief of a special type of Partially Observable Markov Decision Process (POMDP) to which the collaborative Decentralized POMDP can be reduced \cite{Dibangoye13IJCAI,MacDermed13NIPS26,Nayyar13TAC}. 
This enabled tackling these problems using solution methods for POMDPs, leading to increases in scalability \cite{Dibangoye13IJCAI}. 

This paper provides a theoretical basis for enabling similar advancements for zs-POSGs.
In particular, we extend results for Dec-POMDPs to the zs-POSG setting by presenting three contributions: 
	\begin{enumerate}
\item 
   Two novel formulations of the value functions in POSGs, one based on past joint policies, and one based on distributions of information in the game called \emph{plan-time sufficient statistics}.
\item 
    A proof that the latter formulation allows for a generalization over the 
    statistics: on every stage, the value function exhibits concavity and convexity in 
    different \emph{subspaces} of statistic space.
\item 
    A reduction of the zs-POSG to a \emph{Non-Observable Stochastic Game}, 
    which in turn allows us to shows that certain properties previously proven for 
    narrower classes of games generalize to the more general zs-POSG considered
    here.
\end{enumerate}
This is the first work that gives insight in how the value function of a zs-POSG generalizes over the space of sufficient statistics.
We argue that this result may open up the route for new solution methods.

\section{Background}
\label{sec:background}
In this section we provide the necessary background to explain our contributions. 
We defer a treatment of related work to
Section \ref{sec:related_work}, where we can more concisely point out the differences to our work.

This paper focuses on zero-sum games of incomplete information where the number of states, actions, observations and the horizon are finite.
We examine games where the hidden state is static, and games with dynamic state (i.e., it changes over time).
We assume \emph{perfect recall}, i.e., agents recall their own past actions and observations, 
and assume that all elements of the game are \emph{common knowledge} among the agents \cite[Chapter~5]{Osborne+Rubinstein94}.

\subsection{Zero-Sum One-Shot Games}
\label{sec:background:nfgs}
We start by describing one-shot (static) games. 
\begin{definition} A \textbf{normal form game} (NFG) is a tuple 
    $ \mathcal{N} = \langle I, \mathcal{A}, \mathcal{R}\rangle$:
\begin{itemize}
\item $I = \{1, 2\}$ is the set of 2 agents,
\item $\mathcal{A} = \mathcal{A}_1 \times \mathcal{A}_2$ is the set of joint actions $a = \langle a_1, a_2 \rangle$, also called (joint) strategies,
\item $\mathcal{R} = \{R_1, R_2\}$ is the set of reward (or `payoff') functions
    for the agents: $R_i : \mathcal{A} \rightarrow \mathbb{R}$ is the reward
    function for agent~i,
\end{itemize}
\end{definition}
\noindent 
In the case of a \emph{zero-sum} NFG (zs-NFG), we have that $R_1(a) = -R_2(a), \forall a$. In a zs-NFG, 
we will define the following quantities (and the associated strategies):
\begin{definition}
    The \textbf{maxmin value} (for agent 1) of a 2 player zero-sum game 
    is defined as
    $
    V_{maxmin}(\mathcal{N}) = \max_{a_1} \min_{a_2} R_1(a_1, a_2)
    $.
\end{definition}
\begin{definition}
    The \textbf{minmax value} (for agent 1) of a 2 player zero-sum game is defined as
    $
    V_{minmax}(\mathcal{N}) =  \min_{a_2} \max_{a_1} R_1(a_1, a_2)
    $.
\end{definition}

The min-max theorem states that  
$\forall{\mathcal{N}}\;\; V_{maxmin}(\mathcal{N}) \leq V_{minmax}(\mathcal{N})$ 
\cite{Boyd04book}. 
In case of equality, we say that 
$V(\mathcal{N}) = V_{maxmin}(\mathcal{N}) = V_{minmax}(\mathcal{N})$ is the
\textbf{value} of the game. 

A \emph{Nash Equilibrium} (NE) is a joint strategy from which no agent has an
incentive to unilaterally deviate. In zs-NFGs, if an NE exists, then it coincides with the
value of the game.  
That is, the value of a game is the value attained when
both agents follow the strategy specified by the NE. Moreover any NEs, also called saddle points in this context,
correspond to maxmin-strategies for the players \cite[Proposition 22.2]{Osborne+Rubinstein94}. 
We will refer to such strategies as \emph{rational} strategies.

\providecommand{\Mu}{M}
The value is not guaranteed to exist in all zs-NFGs. If the action sets are
convex and compact, if 
$\forall a_2, \text{ the mapping } a_1 \to R_1(a_1, a_2)$ is a concave function,
and if
 $\forall a_1,\; a_2 \to R_1(a_1, a_2)$ is a convex function,
then the value exists and the zs-NFG has a  Nash equilibrium (i.e., a saddle point) \cite[p. 134]{aubin1998optima}, \cite[Proposition 20.3]{Osborne+Rubinstein94}.
In games where actions are discrete, an NE (and corresponding value) may not exist. 
However, when mixed strategies are allowed, we can convert
to a continuous action NFG: 
$\langle I, \Mu_1 \times \Mu_2, \mathcal{R}' \rangle$, where
$\mu_i \in \Mu_i$ specifies a probability distribution over actions, and where
$R_i'(\mu_1, \mu_2) = \sum_{a_1} \mu_1(a_1) \sum_{a_2} \mu_2(a_2) R_i(a_1,a_2)$.
Furthermore, it can now be shown that the utility function $R'(\mu_1, \mu_2)$
in terms of mixed strategies is concave in $\mu_1$ for each $\mu_2$, and convex
in $\mu_2$ for each $\mu_1$, such that the aforementioned assumptions needed for the existence 
of a saddle-point hold
(i.e., an NE exists \cite[p.239]{Boyd04book}).

\subsection{Zero-Sum Bayesian Games}
\label{sec:background:bgs}
Here we consider  \emph{zero-sum Bayesian Games} (zs-BGs), in which agents
simultaneously select an action based on an individual observation (often
referred to as their \emph{type}).

\begin{definition}
A \textbf{zs-BG} is defined as a tuple
 $ \bgmodel = \langle I, \Theta, \mathcal{A}, R, \sigma \rangle$:
\begin{itemize}
\item $I = \{1, 2\}$ is the set of 2 agents,
\item $\Theta = \Theta_1 \times \Theta_2$ is the finite set of joint types $\theta = \langle \theta_1, \theta_2 \rangle$,
\item $\mathcal{A} = \mathcal{A}_1 \times \mathcal{A}_2$ is the finite set of joint actions $a = \langle a_1, a_2 \rangle$,
\item $R: \Theta \times \mathcal{A} \rightarrow \mathbb{R}$ is the reward function for agent 1,
\item $\sigma \in \Delta(\Theta)$ is the probability distribution over joint types.
\end{itemize}
\end{definition}
In this paper we treat finite zs-BGs, where the sets of actions and types are finite. 
A pure strategy, to which we refer as a \emph{pure decision rule}, is a mapping from types to actions.  
A \emph{stochastic decision rule} $\delta_i \in \sdas i$ is a mapping from types to probability distributions over the set of actions, denoted as $\delta_i(a_i | \theta_i)$, $\sdas{i}$ is the space of such mappings.
Given a joint decision rule $\delta = \langle \delta_1, \delta_2 \rangle$, the value is:

\begin{IEEEeqnarray}{rCl}
Q_\text{BG}(\bgmodel, \delta) \triangleq 
		\sum_{\theta} \sigma(\theta) 
			\sum_{a} 
				\delta(a | \theta) R(\theta, a).
\label{eq:bg:qvalue-definition}
\end{IEEEeqnarray}
where $\delta(a|\theta)\triangleq\delta_{1}(a_{1}|\theta_{1})\delta_{2}(a_{2}|\theta_{2})$.
The case of pure decision rules is covered by treating them as degenerate stochastic policies. 

There are two ways to reduce a zs-BG to a zs-NFG. First, we can simply reinterpret $Q_\text{BG}(\bgmodel,\delta)$ 
as a payoff function  $R(\delta_1,\delta_2)$, such that the BG corresponds to
an NFG with continuous action sets $\sdas i$ for each player $i$. 
The conditions for the existence of the value for such a game can be shown to hold, 
so if we suppose that $\langle\delta_{1}^*,\delta_{2}^*\rangle$ is an NE, then the value of the game
can be defined as the maxmin (=minmax) value:

\begin{IEEEeqnarray}{rCl}
	\IEEEeqnarraymulticol{3}{l}{
V_\text{BG}(\bgmodel)	
\triangleq
        Q_\text{BG}(\bgmodel,\langle\delta_{1}^*,\delta_{2}^*\rangle) 
=
} \nonumber  \\ 
    \max_{\delta_{1}\in\Delta_{1}^{S}}\min_{\delta_{2}\in\Delta_{2}^{S}}
        Q_\text{BG}(\bgmodel,\langle\delta_{1},\delta_{2}\rangle) 
=
    \min_{\delta_{2}\in\Delta_{2}^{S}}\max_{\delta_{1}\in\Delta_{1}^{S}}
        Q_\text{BG}(\bgmodel,\langle\delta_{1},\delta_{2}\rangle) . \hspace{5mm}
\label{eq:bg:value-definition}
\end{IEEEeqnarray}

Alternatively, one can reinterpret the BG as an NFG with finite action sets corresponding
to the pure decision rules. This then leads to a similar $R'(\mu_1, \mu_2)$ formulation where the 
mixed strategies $\mu_i$ now are distributions over pure decision rules.
Again this formulation will satisfy the required assumptions on
concavity/convexity, such that this reinterpretation leads
to the same logical conclusion that the value of the finite zs-BG exists
(these dual perspectives are possible due to the one-one correspondence between stochastic decision
rules and mixed strategies).

These reductions imply that solution methods for zs-NFGs (e.g., via linear programming \cite{Shoham08book}) can be used to solve zs-BGs. 
However, such an approach does not scale well --- a more efficient solution method is to convert the zs-BG to \emph{sequence
form} \cite{koller1994fast}.

\subsection{Zero-sum POSGs}
\label{sec:background:posgs}
A zero-sum Partially Observable Stochastic Game (zs-POSG) is a model for multi-agent decision making under uncertainty in zero-sum sequential games where the state changes over time, and the agents simultaneously choose actions at every stage.
\begin{definition}
    \label{def:POSG}
A \textbf{finite zs-POSG} is defined as a tuple 
$\posgmodel = \langle h, I, \mathcal{S}, \mathcal{A}, \mathcal{O}, T, O, R, b^0 \rangle$:
\begin{itemize}
\item $h$ is the horizon,
\item $I = \{1, 2\}$ is the set of $2$ agents,
\item $\mathcal{S}$ is the finite set of states $s$,
\item $\mathcal{A} = \mathcal{A}_1 \times \mathcal{A}_2$ is the finite set of joint actions $a = \langle a_1, a_2 \rangle$,
\item $\mathcal{O} = \mathcal{O}_1 \times \mathcal{O}_2$ is the finite set of joint observations $o = \langle o_1,  o_2 \rangle$,
\item $T$ is the transition function Pr($s^{t+1} | s^t, \act{t} )$,
\item $O$ is the observation function Pr($o^{t+1} | s^{t+1}, \act{t})$,
\item $R : S \times \mathcal{A} \times S \rightarrow \mathbb{R}$ is the reward function for agent 1,
\item $b^0 \in \Delta(\mathcal{S})$ is the probability distribution over states.
\end{itemize}
\end{definition}
In the zs-POSG, we aim to find maxmin-strategies
and corresponding value.
Let a \emph{pure policy} for agent $i$ be a mapping from individual action-observation histories (AOHs)
$\AOH[i]{t} =  \langle \act[i]{0}, \obs[i]{1}, \ldots, \act[i]{t-1}, \obs[i]{t} \rangle $ to actions. 
Let a \emph{stochastic policy} for agent $i$
be a mapping from individual AOHs
to a probability distribution over actions, denoted as
$\pi_i(\act[i]{t} | \AOH[i]{t} )$.
An individual policy defines action selection of one agent on every stage of the game, and is essentially a sequence of individual decision rules
$\pi_i = \langle \dr[i]{0} \ldots \dr[i]{h-1} \rangle$.
We define the \emph{past individual policy} as a tuple of decision rules $\pjp[i]{t} = \langle \dr[i]{0}, \ldots, \dr[i]{t-1} \rangle$, and define the tuple containing decision rules from stage $t$ to $h$ as the \emph{partial individual policy} $\ppol[i]{t} = \langle \dr[i]{t}, \ldots, \dr[i]{h-1} \rangle$.

As in zs-BGs, it is theoretically possible to convert a zs-POSG to a zs-NFG and solve using standard methods, but this is infeasible in practice. An alternative is to converting the zs-POSG to an extensive form game (EFG) and solve it in sequence form \cite{koller1994fast}. While this is more efficient than the NFG route, it is still intractable: the resulting EFG is huge since its size depends on the number of full histories (trajectories of joint actions, joint observations, and
states) \cite{oliehoek2006dec}.

With the goal of opening paths to completely new approaches of tackling
zs-POSGs, this paper focuses on the description of the \textbf{value function}
of zs-POSGs, which (in contrast to the \emph{value} of a game as defined above)
is a function mapping from some notion of `state' of a game to the expected
value (for agent 1). Similar to how the `value' is defined by rational
strategies, we will use the term `value function' for a function that captures
the future expected rewards under a rational joint policy. However, since one
can also reason about the value of non-rational policies, we will refer to the
`rational value function' if there is a need to clarify.

\section{Structure in One-Shot Value}
\label{sec:family_of_bgs}
In order to provide a value function definition for the
sequential setting (in Section \ref{sec:posgs}), we will rely on an intermediate
result for one-shot games developed in this section. 
In particular, we introduce the concept of a Family of zero-sum Bayesian Games, 
for which we will introduce the joint type distribution as a suitable notion of `state' and prove that its value function exhibits certain concave/convex properties with respect to this notion.

\subsection{Families of Bayesian Games}
Here we consider the notion of a \emph{family} of (zs)-BGs. Intuitively,
different stages $t$ of a POSG are similar to a BG: each agent has a privately
observed history, which corresponds to its type. However, the probabilities of
these histories might depend on how the game was played in earlier stages.
As such, we will need to reason about families of BGs.
\newcommand{\fobg}{\mathcal{F}}
\begin{definition}
A \textbf{Family of Bayesian Games}, $\fobg = \langle I, \Theta, \mathcal{A}, R \rangle$, is the set of Bayesian Games of the form $\langle I, \Theta, \mathcal{A}, R, \sigma \rangle$ for which $I, \Theta, \mathcal{A}$ and $R$ are identical.
\label{def:fobg}
\end{definition}
\vspace{.5mm}
\noindent Let $\fobg$ be a Family of \emph{zero-sum} Bayesian Games.
By providing a joint type distribution, $\fobg(\sigma)$ indicates a particular zs-BG.
We generalize \eqref{eq:bg:qvalue-definition} and \eqref{eq:bg:value-definition} as follows:

\begin{IEEEeqnarray}{rCl}
\Q{\fobg}(\sigma,\delta) &\triangleq&  Q_\text{BG}({\fobg}(\sigma) , \delta ),
\label{eq:fobg:qvalue-definition} \\
\optV{\fobg}(\sigma) &\triangleq&  V_\text{BG}({\fobg}(\sigma)).
\label{eq:fobg}
\end{IEEEeqnarray}
As each ${\fobg}(\sigma)$ is a zs-BG, the rational value function $\optV{\fobg}(\sigma)$ can be written as:

\begin{IEEEeqnarray}{rCl}
\optV{\fobg}(\sigma) 
    = 
		\max\limits_{\delta_1 \in \sdas{1}} 
		\min\limits_{\delta_2 \in \sdas{2}} 
		\Q{\fobg}(\sigma, \langle \delta_1, \delta_2 \rangle).
\label{eq:fobg:redefinition}
\end{IEEEeqnarray}
We define \emph{best-response value functions} that give the best-response value to a decision rule of the opposing agent:

\begin{IEEEeqnarray}{rCl}
\V[\text{BR1}]{\fobg}(\sigma, \delta_2) &\triangleq&
   \max\limits_{\delta_1 \in \sdas{1}} 
	\Q{\fobg}(\sigma, \langle \delta_1, \delta_2 \rangle) , \
\label{eq:fobg:bestresponsedef:max} \\
\V[\text{BR2}]{\fobg}(\sigma, \delta_1) &\triangleq&
   \min\limits_{\delta_2 \in \sdas{2}} 
	\Q[]{\fobg}(\sigma, \langle \delta_1, \delta_2 \rangle) .
\label{eq:fobg:bestresponsedef:min}
\end{IEEEeqnarray}
We remind the reader that, for each $\sigma$, 
$\Q[]{\fobg}(\sigma, \langle \delta_1, \delta_2 \rangle)$ 
exhibits concavity and convexity in the space of \emph{decision rules}, as
discussed in Section \ref{sec:background:bgs}.

\subsection{Concavity/Convexity of the Value Function}
\label{sec:family_of_bgs:concavity-convexity}
We saw in the previous section that the concave/convex shape of the utility
function expressed in the space of decision rules follows directly from known
results.  Here we provide a novel result: $\optV{\fobg}$ exhibits a similar
concave/convex shape \emph{in terms of type distributions}. This is important,
as it provides insight on how the \emph{distribution of information} (in
addition to the distribution of actions) affects the value of the game, which
is key to enabling generalization of the value in different parts of sequential
games.
To provide this formulation, we decompose $\sigma$ into a \emph{marginal} term
$\sigma_{m,i}$ and a \emph{conditional} term $\sigma_{c,i}$, here shown for
$i=1$:

\begin{IEEEeqnarray}{rCl}
\sigma_{m,1}(\theta_1) 
	&\triangleq &
		\sum_{\theta_2 \in \Theta_2} \sigma(\theta_1 \theta_2) , \quad \quad
\sigma_{c,1}( \theta_2 | \theta_1) 
	\triangleq 
     \frac{ \sigma(\theta_1 \theta_2) }
			 { \sigma_{m,1}(\theta_1) }. \hspace{4mm}
\label{eq:fobg:marginal-conditional-definition}
\end{IEEEeqnarray}

The terms $\sigma_{m,2}$ and $\sigma_{c,2}$ are defined similarly.
We refer to the simplex $\Delta(\Theta_i)$ containing marginals $\sigma_{m,i}$ as the \emph{marginal-space} of agent $i$. 
We write $\stat{t} = \statmarg[i]{t} \statcond[i]{t}$ for short.

We will first show that the best-response value functions defined in \eqref{eq:fobg:bestresponsedef:max} and \eqref{eq:fobg:bestresponsedef:min} are linear in their respective marginal-spaces.
Using this result, we prove that $\optV{\fobg}$ exhibits concavity in $\Delta(\Theta_1)$ for every $\sigma_{c,1}$, and convexity in $\Delta(\Theta_2)$ for every $\sigma_{c,2}$.
For this purpose, let us define a vector that contains the reward for agent 1 for each individual type $\theta_1$, given $\sigma_{c,1}$ and given that agent 2 follows decision rule $\delta_2$:

\begin{equation}   
    ~\hspace{-2mm}
    \vec{r}_{[ \sigma_{c,1}, \delta_2 ]}(\theta_1)
\triangleq
	\hspace{-0.5mm}
    \max\limits_{a_1 \in \mathcal{A}_1} \bigg[ 
    \sum_{\theta_2} \hspace{-.35mm}
    \sigma_{c,1}(\theta_2 | \theta_1)
    \sum_{a_2} 
    \hspace{-.35mm}
    \delta_2(a_2 |\theta_2)	 
    R(\theta, \act{}) \bigg].
    ~\hspace{-2mm}
    \label{eq:fobg:rvector}
\end{equation}

\noindent The vector $\vec{r}_{[ \sigma_{c,2}, \delta_1 ]}$ is defined analogously. 

Now we can state an important lemma that shows that the best-response value
functions are linear functions in the marginal space.

\begin{lemma}
(1) $\V[\text{BR1}]{\fobg}$ is linear in $\Delta(\Theta_1)$ for all $\sigma_{c,1}$ and $\delta_2$, and (2) $\V[\text{BR2}]{\fobg}$ is linear in $\Delta(\Theta_2)$ for all $\sigma_{c,2}$ and $\delta_1$: 
\begin{enumerate}
\item \inlineequation[eq:lemma:VBR_linear_in_statspace_FOBG:result1]{	
			\V[\text{BR1}]{\fobg}(\sigma_{m,1} \sigma_{c,1}, \delta_2) 
			=  \vec{\sigma}_{m,1} \cdot \vec{r}_{[ \sigma_{c,1}, \delta_2 ]}, \hfill
		}
\item \inlineequation[eq:lemma:VBR_linear_in_statspace_FOBG:result2]{
			\V[\text{BR2}]{\fobg}(\sigma_{m,2} \sigma_{c,2}, \delta_1) 
			=  \vec{\sigma}_{m,2} \cdot \vec{r}_{[ \sigma_{c,2}, \delta_1 ]}. \hfill
		}
\end{enumerate}
\label{lemma:VBR_linear_in_statspace_FOBG}
\end{lemma}
\begin{proof}
The proof is listed in Appendix \ref{appendix:vbr_linear_in_statspace}.
\end{proof}

This lemma can now be used to show that the rational value function
$\optV[]{\fobg}$ possess concave and convex structure with respect to
(appropriately chosen marginals of) the space of joint type distributions.

\begin{customtheorem}
    \label{thm:BGsConcaveConvex}
$\optV[]{\fobg}$ is (1) concave in $\Delta(\Theta_1)$ for a given conditional distribution $\sigma_{c,1}$, and (2) convex in $\Delta(\Theta_2)$ for a given conditional distribution $\sigma_{c,2}$.
More specifically, $\optV[]{\fobg}$ is respectively a minimization over linear functions in $\Delta(\Theta_1)$ and a maximization over linear functions in $\Delta(\Theta_2)$:
\begin{enumerate}
\item $	\optV[]{\fobg}(\sigma_{m,1} \sigma_{c,1}) 
		= \min\limits_{\delta_2 \in \sdas{2}} 
			\bigg[	
				\vec{\sigma}_{m,1} \cdot \vec{r}_{[\sigma_{c,1}, \delta_2 ]}
			\bigg]$,
\item $	\optV[]{\fobg}(\sigma_{m,2} \sigma_{c,2}) 
		= \max\limits_{\delta_1 \in \sdas{1}} 
			\bigg[ 
				\vec{\sigma}_{m,2} \cdot \vec{r}_{[\sigma_{c,2}, \delta_1 ]}
			\bigg]$.
\end{enumerate}
\label{the:fobg:concave_and_convex_value_function}
\end{customtheorem}

\begin{proof}
Filling in the result of Lemma \ref{lemma:VBR_linear_in_statspace_FOBG} gives:

\begin{IEEEeqnarray*}{rCl}
\optV[]{\fobg}(\sigma_{m,1} \sigma_{c,1} ) %= \optV[]{\fobg}(\sigma )
	&\numeq{\ref{eq:fobg:redefinition},
			\ref{eq:fobg:bestresponsedef:max}}&
		 \min\limits_{\delta_2 \in \sdas{2}}
		 	\V[\text{BR1}]{\fobg}(\sigma_{m,1} \sigma_{c,1}, \delta_2) \\
	&\numeq{\ref{eq:lemma:VBR_linear_in_statspace_FOBG:result1}}&
		\min\limits_{\delta_2 \in \sdas{2}} 			
			\bigg[	
				\vec{\sigma}_{m,1} \cdot \vec{r}_{[\sigma_{c,1}, \delta_2 ]}
			\bigg].
\end{IEEEeqnarray*}

\noindent The proof for item 2 is analogous to that of item 1. 
\end{proof}

The importance of this theorem is that it gives direct ways to approximately
generalize the value function using piecewise-linear and convex (concave)
functions.

\section{Structure in zs-POSG Value}
\label{sec:posgs}
Given the result for one-shot games established in the previous section, we are
now in the position to present our main contributions: novel formulations for the
value function, and generalization of the structural result of
Theorem~\ref{thm:BGsConcaveConvex} to the sequential setting.
First, we will introduce a description of the rational value function
based on past joint policies.

\subsection{Past Joint Policies}
\label{sec:posg:past-joint-policies}
To make rational decisions at stage $t$ in the zs-POSG, it is sufficient to know the \emph{past decisions}, which are captured in the past joint policy $\pjp{t}$. 
To show this, the zs-POSG value function in terms of the past joint policy $\pjp{t}$ can be defined in terms of $\pjp{t}$ by extending the formulation by Oliehoek \cite{Oliehoek13IJCAI}.

We define the value function of the zs-POSG at a stage $t$ in terms of a \emph{past joint policy} $\pjp{t}$.
The definition that follows gives the value attained when all agents follow the joint decision rule $\dr{t}$, assuming that in future stages agents will act rationally. That is, the agents follow a rational joint future policy $\ppol[i]{t+1*} = \langle \optdr{t+1} \ldots \optdr{h-1} \rangle$.

We first define the Q-value function at the final stage $t=h-1$, and give an inductive definition of the Q-value function at preceding stages. 
We then define the value function at every stage.
Let immediate reward for a joint AOH and a joint decision rule be defined as:

\begin{IEEEeqnarray*}{rCl}
R( \AOH{t}, \dr{t}) \triangleq 
	 \sum_{\act{t}} \dr{t}(\act{t} | \AOH{t}) 
	 \sum_{\st{t}} 
		\text{Pr} (\st{t}  |  \AOH{t}, b^0)
        R(\st{t}, \act{t}) .   \hspace{5mm}
\end{IEEEeqnarray*} 

\noindent For the final stage $t=h-1$, the Q-value function reduces to this immediate reward, as there is no future value:

\begin{IEEEeqnarray}{rCl}
\optQ{h-1}(\pjp{h-1}, \AOH{h-1}, \dr{h-1}) \triangleq R(\AOH{h-1}, \dr{h-1}).
\label{eq:posg:value-AOH-final-stage}
\end{IEEEeqnarray}

\noindent 
Given an AOH and decision rule at stage $t$, it is possible to find a probability distribution over joint AOHs at the next stage, as a joint AOH at $t+1$ is the joint AOH at $t$ concatenated with joint action $\act{t}$ and joint observation $\obs{t+1}$:

\begin{IEEEeqnarray}{rCl}
\text{Pr} (\AOH{t+1}  | \AOH{t}, \dr{t})   \nonumber
	    \hspace{-0.5mm}
 	&=& 
 		\hspace{-0.5mm} 
 	\text{Pr}( \langle \AOH{t}, \act{t}, \obs{t+1} \rangle | \AOH{t}, \dr{t} ) 
 \hspace{-0.5mm} =  \hspace{-0.5mm} 
	 \text{Pr}( \obs{t+1}  | \AOH{t}, \act{t}) \dr{t}(\act{t} | \AOH{t}) .
\label{eq:posg:probability_AOH_update}
\end{IEEEeqnarray} 

\noindent
For all stages except the final stage $t=h-1$, the value at future stages is propagated to the current stage as follows:

\begin{IEEEeqnarray}{rCl}
\optQ{t}(\pjp{t}, \AOH{t}, \dr{t})	&\triangleq&
	 R(\AOH{t}, \dr{t}) +
        \sum_{\act{t} \in \mathcal{A}} \sum_{\obs{t+1}} 
            \text{Pr}(\AOH{t+1}  |  \AOH{t}, \dr{t}) \nonumber \\      
    &&        \optQ{t+1}(\pjp{t+1}, \AOH{t+1}, \optdr{t+1}).
	\label{eq:posg:value-AOH-pjp}
\vspace{2mm} \\ 
\optQ{t}(\pjp{t}, \dr{t}) &\triangleq&
		\sum_{\AOH{t} \in \AOHs{t}} 
		    \text{Pr} (\AOH{t}  |  b^0, \pjp{t})
		    \optQ{t}(\pjp{t}, \AOH{t}, \dr{t}).
\label{eq:posg:qvalue-pjp}
\end{IEEEeqnarray}

\noindent
We use \eqref{eq:posg:qvalue-pjp} to find rational decision rules for both agents.
Consistent with \eqref{eq:posg:value-AOH-pjp}, we show how to find $\optdr{t+1} = \langle \optdr[1]{t+1}, \optdr[2]{t+1} \rangle$: 

\begin{IEEEeqnarray}{rCl}
\optdr[1]{t+1} &=& \argmax\limits_{\dr[1]{t+1} \in \sdas{1}}  
                     \min\limits_{\dr[2]{t+1} \in \sdas{2}} 
                      \optQ{t+1}(\pjp{t+1}, 
                      			\jointdr{t+1} )  ,
\label{eq:posg:decision-selection-value-pjp1}\\ 
\optdr[2]{t+1} &=& \argmin\limits_{\dr[2]{t+1} \in \sdas{2}}  
                     \max\limits_{\dr[1]{t+1} \in \sdas{1}} 
                      \optQ{t+1}(\pjp{t+1}, 
                      			\jointdr{t+1} ) .
\label{eq:posg:decision-selection-value-pjp2}
\end{IEEEeqnarray}

\noindent
Using \eqref{eq:posg:value-AOH-final-stage}, \eqref{eq:posg:decision-selection-value-pjp1} and \eqref{eq:posg:decision-selection-value-pjp2}, a rational joint decision rule $\optdr{h-1}$ can be found by performing a maxminimization over immediate reward.
Evaluation of $\optQ{h-1}(\pjp{t}, \optdr{h-1})$ gives us the value at stage $t=h-1$, and
\eqref{eq:posg:value-AOH-pjp} propagates the value to the preceding stages. 
As such, rationality for all stages follows by induction. 
The value function can now be defined in terms of the past joint policy as:

\begin{IEEEeqnarray}{rCl}
\optV{t}(\pjp{t}) = 
	\max\limits_{\dr[1]{t} \in \sdas{1}} 
	\min\limits_{\dr[2]{t} \in \sdas{2}}
			\optQ{t}(\pjp{t}, \jointdr{t}).
\label{eq:posg:value-pjp}
\end{IEEEeqnarray}

\noindent 
By \eqref{eq:posg:value-AOH-final-stage}, \eqref{eq:posg:decision-selection-value-pjp1} and \eqref{eq:posg:decision-selection-value-pjp2}, $\optdr{t}$ is dependent on $\optdr{t+1}$, and thus on the rational future joint policy.
However, $\optdr{t+1}$ can only be found if past joint policy $\pjp{t+1}$, which includes $\dr{t}$, is known.
As such, \eqref{eq:posg:value-pjp} 	can not be used to form a backward inductive approach directly.

\subsection{Plan-Time Sufficient Statistics}
\label{sec:posg:plantime-sufficient-statistics}
A disadvantage of the value function definition from the previous section is that the value at a stage $t$ depends on all the \emph{joint} decisions from stage 0 to $t$.
We propose to define the value function in terms of a plan-time sufficient statistic that summarizes many past joint policies.
Furthermore, we give formal proof that this value function exhibits a concave/convex shape in statistic-space that may be exploitable.

\begin{definition} 
The \textbf{plan-time sufficient statistic} for a general past joint policy $\pjp{t}$, assuming $b^0$ is known,
is a distribution over joint AOHs:
$
\stat{t}(\AOH{t}) \triangleq 
       \text{Pr}(\AOH{t}  |  b^0, \pjp{t} )
$. 
\label{def:plantimesufficientstatistic}
\end{definition}

In the collaborative Dec-POMDP case, a plan-time sufficient statistic fully captures the influence of the past joint policy.
We prove that this also holds for the zs-POSG case, thereby validating the previous definition and the name `sufficient statistic`.
We aim to express the value for a given decision rule $\dr{t}$ in terms of a plan-time sufficient statistic, given that the agents act rationally at later stages. 
First, we define the update rule for plan-time sufficient statistics:

\begin{IEEEeqnarray}{rCl}
\stat{t+1}(\AOH{t+1})
    \triangleq \text{Pr}(\obs{t+1}  |  \AOH{t}, \act{t})
      \dr{t}(\act{t}  |  \AOH{t}) 
      \stat{t}(\AOH{t}).
\label{eq:posg:statistic-update-rule}
\end{IEEEeqnarray}

\noindent
 The Q-value function at stage $h-1$ reduces to the immediate reward:

\begin{IEEEeqnarray}{rCl}
\optQ{h-1}(\stat{h-1}, \AOH{h-1}, \dr{h-1}) \triangleq R(\AOH{h-1}, \dr{h-1}).
\label{eq:posg:value-AOH-final-stage-ss}
\end{IEEEeqnarray}

\noindent 
We then define the Q-value for all other stages similar to \eqref{eq:posg:value-AOH-pjp}, \eqref{eq:posg:qvalue-pjp}:

\begin{IEEEeqnarray}{rCl}
\optQ{t}(\stat{t}, \AOH{t}, \dr{t}) &\triangleq&
        R(\AOH{t}, \dr{t}) + 
        \sum_{\act{t}} \sum_{\obs{t+1}} 
            \text{Pr}(\AOH{t+1}  |  \AOH{t}, \dr{t})
    \nonumber \\  &&\,
            \optQ{t+1}(\stat{t+1}, \AOH{t+1}, \optdr{t+1}). \hspace{5mm}
\label{eq:posg:value-AOH-ss} 
\vspace{2mm}\\ 
\optQ{t}(\stat{t}, \dr{t}) &\triangleq&
		\sum_{\AOH{t}} 
		    \stat{t}(\AOH{t})
		    \optQ{t}(\stat{t}, \AOH{t}, \dr{t}).
\label{eq:posg:qvalue-ss}
\end{IEEEeqnarray}

\noindent Rational decision rules can be found using \eqref{eq:posg:qvalue-ss}:

\begin{IEEEeqnarray}{rCl}
\optdr[1]{t+1} &=& \argmax\limits_{\dr[1]{t+1}\in\sdas{1}} 
                     \min\limits_{\dr[2]{t+1}\in\sdas{2}} 
                      \optQ{t+1}(\stat{t+1}, 
                                 \jointdr{t+1}), \hspace{5mm} 
\label{eq:posg:decision-selection-value-ss1} \\
\optdr[2]{t+1} &=& \argmin\limits_{\dr[2]{t+1}\in\sdas{2}} 
                     \max\limits_{\dr[1]{t+1}\in\sdas{1}} 
                      \optQ{t+1}(\stat{t+1}, 
                               \jointdr{t+1}). \hspace{5mm} 
\label{eq:posg:decision-selection-value-ss2}
\end{IEEEeqnarray} 

\noindent 
We formally prove that the statistic $\stat{t}$ provides sufficient information for rational decision-making in the zs-POSG.

\begin{lemma}
$\stat{t}$ is a sufficient statistic for the value of the zs-POSG, i.e.
$\optQ{t}(\stat{t}, \AOH{t}, \dr{t}) = \optQ{t}(\pjp{t}, \AOH{t}, \dr{t}), \forall t \in 0 \ldots h -  1, \forall \AOH{t} \in \AOHs{t}, \forall \dr{t}$.
\label{lemma:sufficiency_statistic}
\end{lemma}%

\begin{proof}
The proof is listed in Appendix \ref{appendix:sufficiency_statistic}.
\end{proof}

\noindent This allows us to define the value function of the zs-POSG
in terms of the sufficient statistic as follows:

\begin{IEEEeqnarray}{rCl}
\optV{t} (\stat{t}) \triangleq 
						\max\limits_{\dr[1]{t} \in \sdas{1}} 
                         \min\limits_{\dr[2]{t} \in \sdas{2}}
                           \optQ{t}(\stat{t}, \jointdr{t}).
\label{eq:posg:definition-value-function-ss}
\end{IEEEeqnarray}

\noindent
Although we have now identified the value at a single stage of the game, implementing a backwards inductive approach directly is still not possible, since the space of statistics is continuous and we do not know how to represent $\optV{t}(\stat{t})$. 
This paper takes a first step at resolving this problem by investigating the structure of $\optV{t}(\stat{t})$.

\subsection{Equivalence Final Stage Zero-Sum POSG and Family of Zero-Sum Bayesian Games}
\label{sec:theoryPOSG:equivalence}
We have already proven that the value function of a Family of zero-sum Bayesian Games exhibits concavity and convexity in terms of the marginal parts of the type distribution $\sigma$ for respectively agent 1 and 2.
We show that the final stage of the zs-POSG $\posgmodel$, $t=h-1$, can be defined as a Family of zs-BGs as follows:

\begin{itemize}
\item $I = \{1,2\}$ is the set of agents,
\item $\Theta = \AOHs{h-1}$ is the set of joint types corresponding to AOHs in zs-POSG $\posgmodel$ at stage $h-1$,
\item $\mathcal{A}$ is the set of joint actions in zs-POSG $\posgmodel$,
\item $R(\AOH{h-1},a)$ follows directly from the immediate reward function of
    $\posgmodel$ (as described in Section~\ref{sec:posg:past-joint-policies}).
\end{itemize}

\noindent 
By the result of Theorem \ref{the:fobg:concave_and_convex_value_function} 
(i.e., that $\optV{\mathcal{F}}$ exhibits concavity and convexity in marginal-spaces of agent 1 and 2 respectively)
the value function at the final stage of the zs-POSG, $\optV{h-1}$, is concave in $\statspace[1]{h-1}$ for all $\statcond[1]{h-1}$, and convex in $\statspace[2]{h-1}$ for all $\statcond[2]{h-1}$. 

Note that, even though the final stage is equal to a Family of Bayesian Games, our approach is substantially different from approaches that represent a POSG as a \emph{series} of BGs \cite{emery2005approximate} and derivative works \cite{oliehoek2008optimal}.
In fact, all other stages (0 to $h-2$) cannot be represented as a Family of BGs,
as the rational value function for stages $t=0,...,h-2$ cannot be expressed as a
function $R(\AOH{t},a)$.

\subsection{Concavity/Convexity of the Value Function}
\label{sec:concavity-convexity-of-the-value-function}
We continue to show that the value function for any stage exhibits the same type of structure. 
In particular, the plan-time sufficient statistic can be decomposed in marginals and conditionals,
and the value function is concave in marginal-space for agent 1, $\statspace[1]{t}$, and convex in marginal-space for agent 2, $\statspace[2]{t}$.
Figure \ref{fig:abstract_representation_conditionalslice} provides intuition on how the
best-response value functions relate to the concave/convex value function:
a `slice' in statistic-space corresponds to a single conditional $\statcond[1]{t}$
($\statcond[2]{t}$) and exhibits concave (convex) shape of the value function made up by linear segments that each correspond to a partial policy of the other agent.
As we will show, each segment corresponds exactly to a best-response value function.

\begin{figure}
    \begin{subfigure}[b]{0.235\textwidth}
    \tdplotsetmaincoords{78}{45}
        \begin{tikzpicture}[tdplot_main_coords]
            \def\widthx{2.65}
            \def\widthy{2.65}
            \def\heightz{1.9}
            
            % draw axes 
            \draw[]	node(stat) at (0.75 * \widthx, 0.25 * \widthy, 0){ \scriptsize $\stat{t}$-space};
            %														, `statistic-space'};
            %\draw[thin, black, ->] ($(stat) + (0.1, 0, -0.15)$) -- (1, 1.5, 0);  
            
            \draw[thick,] (0,0,0) -- (\widthx,0,0) 
            	node[anchor=north] at (0.5 * \widthx, 0, 0)
            			{ \scriptsize $\statcond[1]{t}$-space};
            			%`marginal-space'};
            \draw[thick,anchor=west] (0,0,0) -- (0, \widthy, 0)
            	node[] at (\widthx + 0.25, 0.5 * \widthy-.5, -.1)
            			{  \scriptsize $\statmarg[1]{t}$-space};
            			%`conditional-space'};
            	
            \draw[thick,] (0,\widthy,0) -- (\widthx, \widthy, 0);
            \draw[thick,] (\widthx,0,0) -- (\widthx, \widthy, 0);
            	
            \draw[thin,->] (0,      0,      0) -- (0,      0,      \heightz) node[anchor=south]{$\V{t}$};
            \draw[thin,->] (0,      \widthy,0) -- (0,      \widthy,\heightz) ;
            %\draw[thin,->] (\widthx,0,      0) -- (\widthx,0,      \heightz) ;
            \draw[thin,->] (\widthx,\widthy,0) -- (\widthx,\widthy,\heightz) ;
            
            % select single slice
            \def\sliceX{0.4 * \widthx}
            \draw[dotted, thick] (\sliceX, 0, \heightz) -- (\sliceX, 0, 0) -- (\sliceX, \widthy, 0) -- (\sliceX, \widthy, .9*\heightz);
            
            % draw rotated bestresponsevaluefunction
            \draw[] node[rotate=0,%46., 
            				 scale=0.9] at 
            		(\sliceX, 1.05*\widthy, 0.995 * \heightz) 
            	{  \scriptsize  $\V[\text{BR1}]{t}(\stat{t}, \ppol[2]{t})$};

            \coordinate (f1-p1) at (\sliceX, 0,        0.2 *\heightz);
            \coordinate (f1-p2) at (\sliceX, \widthy,  0.9 *\heightz);
            \coordinate (f2-p1) at (\sliceX, 0,        0.35 *\heightz) ;
            \coordinate (f2-p2) at (\sliceX, \widthy,  0.55 *\heightz);
            \coordinate (f3-p1) at (\sliceX, 0,        0.55 *\heightz);
            \coordinate (f3-p2) at (\sliceX, \widthy,  0.3 *\heightz);
            \coordinate (f4-p1) at (\sliceX, 0,        1.0 *\heightz);
            \coordinate (f4-p2) at (\sliceX, \widthy,  0.05 *\heightz);
            
            \foreach \n in {1,...,4} {
            	\draw[blue] (f\n-p1) -- (f\n-p2);
            }
            
            \def\voffset{-0.05}
            \draw[red, thick] 
            	($(f1-p1) + (0,0,\voffset)$) -- 
            	($(intersection of {f1-p1}--{f1-p2} and {f2-p1}--{f2-p2}) + (0,0,\voffset)$) --
            	($(intersection of {f2-p1}--{f2-p2} and {f3-p1}--{f3-p2}) + (0,0,\voffset)$) --
            	($(intersection of {f3-p1}--{f3-p2} and {f4-p1}--{f4-p2}) + (0,0,\voffset)$) --
            	($(f4-p2) + (0,0,\voffset)$);
            
            % draw aux lines and dots for evaluated statistics
            \coordinate (s1) at (\sliceX, 0.1 * \widthy,   0*\heightz);
            \coordinate (s2) at (\sliceX, 0.34 * \widthy,   0 *\heightz);
            \coordinate (s3) at (\sliceX, 0.55 * \widthy,  0 *\heightz) ;
            \coordinate (s4) at (\sliceX, 0.75 * \widthy,  0 *\heightz);
            \coordinate (sz1) at (\sliceX, 0.1 * \widthy,  1 *\heightz);
            \coordinate (sz2) at (\sliceX, 0.34 * \widthy,, 1 *\heightz);
            \coordinate (sz3) at (\sliceX, 0.55 * \widthy, 1 *\heightz) ;
            \coordinate (sz4) at (\sliceX, 0.75 * \widthy, 1 *\heightz);
            
           % \foreach \s in {1,...,4} {
           % 	\draw[] node[circle,draw,inner sep=.7,fill=black] 
            %		(s\s-isect) at (intersection of {s\s}--{sz\s} and {f\s-p1}--{f\s-p2}) {} ; %
  %          	\draw[dotted, black] (s\s) -- (s\s-isect);
    %        	\draw[] node[circle,draw,inner sep=.7,fill=black]%,label=below:{$\stat[(\s)]{t}$}] 
     %       	        at (s\s) {};
      %      }
        \end{tikzpicture}
    \end{subfigure}%
    \begin{subfigure}[b]{0.5\textwidth}
    \tdplotsetmaincoords{78}{45}
        \begin{tikzpicture}[tdplot_main_coords]
            \def\widthx{2.65}
            \def\widthy{2.65}
            \def\heightz{1.9}
            
            % draw axes 
            \draw[]	node(stat) at (0.75 * \widthx, 0.25 * \widthy, 0){ \scriptsize $\stat{t}$-space};
            %														, `statistic-space'};
            %\draw[thin, black, ->] ($(stat) + (0.1, 0, -0.15)$) -- (1, 1.5, 0);  
            
            \draw[thick,] (0,0,0) -- (\widthx,0,0) 
            	node[anchor=north] at (0.5 * \widthx, -.2, .1)
            			{ \scriptsize $\statmarg[2]{t}$-space};
            			%`marginal-space'};
            \draw[thick,anchor=west] (0,0,0) -- (0, \widthy, 0)
            	node[]             at (     \widthx , 0.5 * \widthy, -0.2)
            			{ \scriptsize $\statcond[2]{t}$-space};
            			%`conditional-space'};
            	
            \draw[thick,] (0,\widthy,0) -- (\widthx, \widthy, 0);
            \draw[thick,] (\widthx,0,0) -- (\widthx, \widthy, 0);
            	
            \draw[thin,->] (0,      0,      0) -- (0,      0,      \heightz) node[anchor=south]{$\V{t}$};
            \draw[thin,->] (0,      \widthy,0) -- (0,      \widthy,\heightz) ;
            %\draw[thin,->] (\widthx,0,      0) -- (\widthx,0,      \heightz) ;
            \draw[thin,->] (\widthx,\widthy,0) -- (\widthx,\widthy,\heightz) ;
            
            % select single slice
            \def\sliceY{0.6 * \widthy}
            \draw[dotted, thick] (0, \sliceY, \heightz) -- (0, \sliceY, 0) -- (\widthx, \sliceY, 0) -- (\widthx, \sliceY, .8*\heightz);
            
            % draw rotated bestresponsevaluefunction
            \draw[] node[rotate=0, %68., 
                         scale=0.9] at 
            	(0.925 * \widthx, \sliceY, 0.9 * \heightz) 
            	{  \scriptsize  $\V[\text{BR2}]{t}(\stat{t}, \ppol[1]{t})$};

            \coordinate (f1-p1) at (0,            \sliceY,  0.7 *\heightz);
            \coordinate (f1-p2) at (0.4*\widthx,  \sliceY,  0.0 *\heightz);
            \coordinate (f2-p1) at (0,            \sliceY,  0.4 *\heightz) ;
            \coordinate (f2-p2) at (\widthx,      \sliceY,  0.1 *\heightz);
            \coordinate (f3-p1) at (0,            \sliceY,  0.1 *\heightz);
            \coordinate (f3-p2) at (\widthx,      \sliceY,  0.4 *\heightz);
            \coordinate (f4-p1) at (0.6*\widthx,  \sliceY,  0.0 *\heightz);
            \coordinate (f4-p2) at (\widthx,      \sliceY,  0.8 *\heightz);
            
            \foreach \n in {1,...,4} {
            	\draw[blue] (f\n-p1) -- (f\n-p2);
            }
            
            \def\voffset{0.05}
            \draw[red, thick] 
            	($(f1-p1) + (0,0,\voffset)$) -- 
            	($(intersection of {f1-p1}--{f1-p2} and {f2-p1}--{f2-p2}) + (0,0,\voffset)$) --
            	($(intersection of {f2-p1}--{f2-p2} and {f3-p1}--{f3-p2}) + (0,0,\voffset)$) --
            	($(intersection of {f3-p1}--{f3-p2} and {f4-p1}--{f4-p2}) + (0,0,\voffset)$) --
            	($(f4-p2) + (0,0,\voffset)$);
            
            % draw aux lines and dots for evaluated statistics
            \coordinate (s1) at (0.1 * \widthx,   \sliceY, 0*\heightz);
            \coordinate (s2) at (0.3 * \widthx,   \sliceY, 0 *\heightz);
            \coordinate (s3) at (0.55 * \widthx,  \sliceY, 0 *\heightz) ;
            \coordinate (s4) at (0.85 * \widthx,  \sliceY, 0 *\heightz);
            \coordinate (sz1) at (0.1 * \widthx,  \sliceY, 1 *\heightz);
            \coordinate (sz2) at (0.3 * \widthx,  \sliceY, 1 *\heightz);
            \coordinate (sz3) at (0.55 * \widthx, \sliceY, 1 *\heightz) ;
            \coordinate (sz4) at (0.85 * \widthx,  \sliceY, 1 *\heightz);
            
            %\foreach \s in {1,...,4} {
            %	\draw[] node[circle,draw,inner sep=.7,fill=black] 
            %		(s\s-isect) at (intersection of {s\s}--{sz\s} and {f\s-p1}--{f\s-p2}) {} ; %
  %          	\draw[dotted, black] (s\s) -- (s\s-isect);
  %          	\draw[] node[circle,draw,inner sep=.7,fill=black]%,label=below:{$\stat[(\s)]{t}$}] 
  %          	    at (s\s) {};
  %          }
        \end{tikzpicture}
    \end{subfigure}
	\caption{
		\footnotesize An abstract visualization of the decomposition of statistic-space into marginal-space and conditional-space. 
	}
\label{fig:abstract_representation_conditionalslice}
\end{figure}
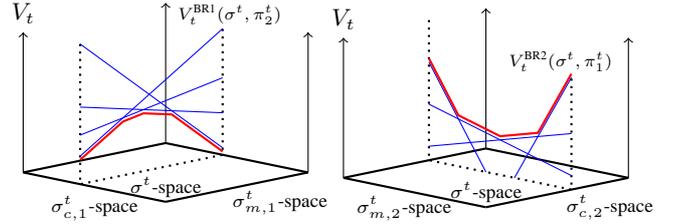
Best-response value functions in terms of $\stat{t}$ and $\ppol[i]{t}$ are defined as $\V[\text{BR1}]{t}$ and
$\V[\text{BR2}]{t}$, similar to \eqref{eq:fobg:bestresponsedef:max} and \eqref{eq:fobg:bestresponsedef:min}.
Let $\spas[i]{t}$ be the space of all stochastic partial policies $\ppol[i]{t}$.
We then have:

\begin{IEEEeqnarray}{rCl}
\optV{t}(\stat{t}) 
	= \min\limits_{\ppol[2]{t} \in \spas[2]{t}} \V[\text{BR1}]{t}(\stat{t}, \ppol[2]{t}) 
	= \max\limits_{\ppol[1]{t} \in \spas[1]{t}} \V[\text{BR2}]{t}(\stat{t}, \ppol[1]{t}).
	\hspace{7mm}
\label{eq:posg:bestresponsedef:derivation}
\end{IEEEeqnarray}

We first show that the best-response value functions $\V[\text{BR}1]{t}$ and $\V[\text{BR}2]{t}$ are linear in their respective marginal-spaces.
Let us define a vector that contains the value (immediate reward \emph{and} future value) for agent
 1 for each individual AOH $\AOH[1]{t}$, given that agent 2 follows the partial policy $\ppol[2]{t}$: 
  
\begin{IEEEeqnarray}{rCl}
	\IEEEeqnarraymulticol{3}{l}{
\vec{\nu}_{[ \statcond[1]{t}, \ppol[2]{t} ]}(\AOH[1]{t})	
	\triangleq 
		\max\limits_{\act[1]{t} \in \mathcal{A}_1} 
\hspace{-0.5mm}
		\bigg[
 			\sum_{\AOH[2]{t} \in \AOHs[2]{t}} \hspace{-1mm}
				\statcond[1]{t}(\AOH[2]{t}|\AOH[1]{t})
			\sum_{\act[2]{t} \in \mathcal{A}_2} \hspace{-1mm}
					\dr[2]{t}(\act[2]{t}|\AOH[2]{t})
} \nonumber  \\ 
	&&	  
	 \bigg( 
	 	  \hspace{-0.7mm}
	 			R(\AOH{t},\act{t})
		 +  \hspace{-1mm}
		\sum\limits_{\obs{t+1} \in \mathcal{O}_1} 
		\hspace{-1mm}
		 	\text{Pr}(\obs{t+1} | \AOH{t}, \act{t})
		 		\vec{\nu}_{[ \statcond[1]{t+1}, \ppol[2]{t+1} ]}
		 			(\AOH[1]{t+1})
		 	\bigg) \bigg]  \hspace{0.7cm}
\label{eq:concaveconvex:valuevector}
\end{IEEEeqnarray}
Note that this is a recursive definition, and that $\AOH[1]{t+1} = \langle \AOH[1]{t}, \act[1]{t}, \obs[1]{t+1} \rangle$.
The vector $\vec{\nu}_{[ \statcond[2]{t}, \ppol[1]{t} ]}$ is defined analogously.

\begin{lemma}
(1) $\V[\text{BR1}]{t}$ is linear in $\statspace[1]{t}$ for a given $\statcond[1]{t}$ and $\ppol[2]{t}$, and (2) $\V[\text{BR2}]{t}$ is linear in $\statspace[2]{t}$ for a given $\statcond[2]{t}$ and $\ppol[1]{t}$, for all stages $t = 0,\ldots, h-1$:
\begin{enumerate}
\item \inlineequation[eq:lemma:VBR_linear_in_statspace:result1]{	
			\V[\text{BR1}]{t}(\statmarg[1]{t} \statcond[1]{t}, \ppol[2]{t}) 
			= \vstatmarg[1]{t} \cdot \vec{\nu}_{[\statcond[1]{t}, \ppol[2]{t} ]}, \hfill
	}
\item \inlineequation[eq:lemma:VBR_linear_in_statspace:result2]{	
		\V[\text{BR2}]{t}(\statmarg[2]{t}\statcond[2]{t}, \ppol[1]{t}) 
		= \vstatmarg[2]{t} \cdot \vec{\nu}_{[\statcond[2]{t}, \ppol[1]{t} ]}. \hfill
	}
\end{enumerate}
\label{lemma:VBR_linear_in_statspace}
\end{lemma}

\begin{proof}
We prove this by induction.
We know the zs-POSG value function at stage $t=h-1$ to be equivalent to that of a Family of zs-BGs, which is concave/convex (Lemma \ref{lemma:VBR_linear_in_statspace_FOBG}).
This is a base case for the proof.
The full proof is listed in Appendix \ref{appendix:vbr_linear_in_statspace}.
\end{proof}

\begin{customtheorem}
$\optV{t}$ is (1) concave in $\statspace[1]{t}$ for a given $\statcond[1]{t}$, and (2) convex in $\statspace[2]{t}$ for a given $\statcond[2]{t}$. 
More specifically, $\optV[]{t}$ is respectively a minimization over linear functions in $\statspace[1]{t}$ and a maximization over linear functions in  $\statspace[2]{t}$:
\begin{enumerate}
\item $	\optV{t}(\statmarg[1]{t} \statcond[1]{t}) 
		= \min\limits_{\ppol[2]{t} \in \spas[2]{t}} 
			\bigg[	
				\vstatmarg[1]{t} \cdot \vec{\nu}_{[\statcond[1]{t}, \ppol[2]{t} ]}
			\bigg]$,
\item $	\optV{t}(\statmarg[2]{t} \statcond[2]{t}) 
		= \max\limits_{\ppol[1]{t} \in \spas[1]{t}} 
			\bigg[ 
				\vstatmarg[2]{t} \cdot \vec{\nu}_{[\statcond[2]{t}, \ppol[1]{t} ]}
			\bigg]$.
\end{enumerate}
\label{the:concave_and_convex_value_function} 	
\end{customtheorem}
\begin{proof}
Filling in the result of Lemma \ref{lemma:VBR_linear_in_statspace} gives:

\begin{IEEEeqnarray*}{rCl}
\optV{t}(\statmarg[1]{t} \statcond[1]{t})
	&\numeq{\ref{eq:posg:bestresponsedef:derivation}}&
		 \min\limits_{\ppol[2]{t} \in \spas[2]{t}} \V[\text{BR1}]{t}(\statmarg[1]{t}\statcond[1]{t}, \ppol[2]{t})  \\ 
	&\numeq{\ref{eq:lemma:VBR_linear_in_statspace:result1}}& 	
		\min\limits_{\ppol[2]{t} \in \spas[2]{t}} 			
			\bigg[	
				\vstatmarg[1]{t} \cdot \vec{\nu}_{[\statcond[1]{t}, \ppol[2]{t} ]}
			\bigg].
\end{IEEEeqnarray*}

\noindent
The proof for item 2 is analogous to that of item 1. 
\end{proof}

The importance of this theorem is that it suggests ways to (approximately) represent $\optV{t}(\stat{t})$. 
Thus, it may enable the development of new solution methods for zs-POSGs. 
To draw the parallel, many POMDP solution methods exploit the fact that a POMDP value function is piecewise-linear and convex (PWLC) in belief-space~\cite{pineau2006anytime,spaan2005perseus} (which is similar to the statistic-space 
we consider),
and recently such results have been extended to the decentralized cooperative (i.e., Dec-POMDP) case \cite{Dibangoye13IJCAI,MacDermed13NIPS26}.

Note that our result is similar to, but different from the saddle-point
function that is typically associated with min-max equilibria. 
In particular,
the value function in is defined in the space of plan-time sufficient
statistics, while the well-known saddle point function
(see Section \ref{sec:background:nfgs}) is a function of complete
strategies. 
This means that the latter is only defined for the reduction to
a normal form game, which per definition destroys any of the specific structure
of the game under concern. 
In contrast, our formulation preserves such structure and thus allows us to make statements about how \emph{value generalizes as a function of the information distribution}.
That is, the formulation allows us to give approximation bounds that generalize within each conditional statistic.

\section{Reduction to NOSG}
\label{sec:posg:reductiontoNOSG}
The application of methods that exploit the PWLC structure of the value function of Dec-POMDPs was enabled by 
a reduction from Dec-POMDP to a 
\emph{non-observable MDP (NOMDP)},
which is a special type of (centralized) POMDP \cite{Dibangoye13IJCAI,MacDermed13NIPS26,Nayyar13TAC,Oliehoek14IASTR}.
This allows POMDP solution methods to be employed in the context of Dec-POMDPs. The proposed plan-time statistics for Dec-POMDPs \cite{Oliehoek13IJCAI} precisely correspond to the belief in the centralized model.

Since we have shown that it is possible to generalize the plan-time statistics to the zs-POSG case, it is reasonable to expect that zs-POSGs can be reduced similarly.
Here we present a reduction from zs-POSG to a special type of stochastic game where information is centralized, to
which we refer as a \emph{Non-Observable Stochastic Game} (NOSG). 
We do not provide the full background of the reduction for the Dec-POMDP case, but refer to \cite{Oliehoek14IASTR}.
The difference between the reduction from Dec-POMDP to NOMDP and the one we present next, is that the zs-POSG is reduced to a stochastic game where the joint AOH acts as the state.

\begin{definition}
A \textbf{plan-time Non-Observable Stochastic Game} for a zs-POSG is a tuple $\langle \dot{S}, \dot{\mathcal{A}}_1, \dot{\mathcal{A}}_1, \dot{\mathcal{O}}, \dot{T}, \dot{O}, \dot{R}, \dot{b}^0 \rangle$:
\begin{itemize}
\item $I = \{1, 2\}$ is the set of agents,
\item $\dot{S}$ is the set of augmented states $\dot{s}^t$, each corresponding to a joint AOH $\AOH{t}$,
\item $\dot{\mathcal{A}} = \dot{\mathcal{A}}_1 \times \dot{\mathcal{A}}_2$ is the continuous action-space, containing stochastic decision rules $\dr{t} = \langle \dr[1]{t}, \dr[2]{t} \rangle$,
\item $\mathcal{\dot{O}} = \{\text{NULL}\}$ is set of joint observations that only contains the NULL observation.
\item $\dot{T}$ is the transition function that specifies 
		$ \dot{T}(\dot{s}^{t+1} \,|\, \dot{s}^t, \langle \dot{a}_1^t, \dot{a}_2^t \rangle )  = 
			\text{Pr}(\AOH{t+1} \,|\, \AOH{t}, \langle \dr[1]{t}, \dr[2]{t} \rangle). $
\item $\dot{O}$ is the observation function that specifies that observation 
		NULL is received with probability 1.
\item $ \dot{\mathcal{R}} : \dot{S} \times \mathcal{A} \rightarrow \mathbb{R}$ is the reward function $R(\AOH{t}, \langle \dr[1]{t}, \dr[2]{t} \rangle) $,
\item $\dot{b}^0 \in \Delta(\dot{S})$ is the initial belief over states.
\end{itemize}
\end{definition}
\vspace{.5mm}

In the NOSG model, agents condition their choices on the joint belief over
augmented states $\dot{b} \in \Delta(\dot{S})$, which corresponds to the belief
over joint AOHs captured in the statistic $\stat{t} \in \Delta(\Theta^t)$.  As
such, a value function formulation for the NOSG can be given in accordance with
\eqref{eq:posg:definition-value-function-ss}.

In order to avoid potential confusion, let us point out that a zs-POSG can also 
be converted to a best response POMDP by fixing the policies of one agent
\cite{nair2003taming}, which leads to a model where the information state
$b(s,\theta_j)$ is a distribution over states and AOHs of the other agent.  In
contrast, our NOSG formulation maintains a belief over only joint AOHs. More
importantly, however, our approach does not require fixing the policy of any
agent, which would necessitate recomputing of all values when the fixed policy
changes (if the past policy changes, the distributions over paths change, and
thus future values are affected).  
Where the approach of Nair et al.~\cite{nair2003taming} leads to a single-agent
model that can be used to compute a best-response, our conversion leads to a
multi-agent model that can be used to compute a Nash equilibrium
directly.  

A key contribution of our NOSG formulation is that it directly indicates that
properties of `zero-sum stochastic games with shared observations'
\cite{ghosh2004zero} also hold for zs-POSGs.

\begin{definition} \label{def:zsSOSG}
    A \textbf{zero-sum Shared Observation Stochastic Game (zs-SOSG)}
    is a zs-POSG (cf. Def.~\ref{def:POSG}) in which each agent receives the
    joint observation and can observe the other's actions.
\end{definition}

\noindent
Ghosh et al. \cite{ghosh2004zero} show that, under some technical assumptions, a
zs-SOSG can be converted into a completely observable model (similar to the
conversion of a POMDP into a belief MDP), and that, in the infinite-horizon
case, both the value and a rational joint policy exists.

Our claim is that these results in fact transfer to the more general class
of zs-POSGs via our NOSG construction. We start by noting:

\begin{lemma}
The plan-time NOSG of a finite-horizon zs-POSG is a zs-SOSG.
\end{lemma}
\begin{proof} 
    In the finite-horizon case, the set of states in our plan-time
    NOSG is discrete. The action space, while continuous, is of finite
    dimensionality, and is a closed and bounded set. The set of shared
    observations only consists of a trivial NULL observation and by assuming
    rationality for both players, we can assume that they observe each others
    actions (corresponding to decision rules of the POSG).
\end{proof}

In the infinite-horizon case, some assumptions on the class of decision rules
$\dr[i]{t}$ are needed to be able to formulate the plan-time NOSG model
and assuming infinite policy trees as the policies would violate certain
technical requirements on which Ghosh's results depend (e.g., the action-spaces are
required to be metric and compact spaces). 
However, a straightforward extension of our reduction for the case where the
agents use finite-state controllers (analogous to such formulations for
Dec-POMDPs \cite{MacDermed13NIPS26}), would satisfy these requirements, and as
such we can infer the existence of a value for such games:
\begin{corollary}
    For infinite-horizon zs-POSGs where agents are restricted to use
    finite-state controllers for their policies, the value of the game exists.
\end{corollary}
\noindent
In that way, our reduction shows that some of the properties established by
Ghosh et al. for a limited subset of zero-sum stochastic games, in fact extend
to a much broader class of problems.

\section{Related Work}
\label{sec:related_work}
There is rich body of literature on zero-sum games, and we provide pointers to
the most relevant works here. 
The concave and convex structure we have found for the zs-POSG value function 
is similar to the saddle point structure associated with min-max equilibria \cite{aubin1998optima}.
Note, however, that we have defined the value function in terms of the distribution over information, rather than the substantially different space of \emph{joint strategies}, solving of which requires flattening the game to normal form.
That is, our results tell us something about the value of acting using the current information, thus they may give insight into games of general partial observability.

A recent paper that is similar in spirit to ours is by Nayyar et al.\ \cite{nayyar2014common} who introduce a so-called Common Information Based Conditional Belief --- a probability distribution over AOHs and the state conditioned on common information --- and use it to design a dynamic-programming approach for zs-POSGs. 
This method converts stages of the zs-POSG to Bayesian Games for which the type distribution corresponds to the statistic at that stage.
However, since their proposed statistic is a distribution over joint AOHs \emph{and states}, the statistic we propose in this paper is more compact. 
More importantly, Nayyar et al.\ do not provide any results regarding the structure of the value function, which is the main contribution of our paper.

Hansen et al. \cite{hansen2004dynamic} present a dynamic-programming approach
for finite-horizon (general sum) POSGs that iteratively constructs
sets of one-step-longer (pure) policies for all agents. At every iteration, the sets of individual policies are pruned by removing dominated policies.  
This pruning is based on a different statistic called \emph{multi-agent belief}: a
distribution over states and policies of other agents. Such a multi-agent belief
is sufficient from the perspective of an individual agent to determine its best
response (or whether some of its policies are dominated).  A more generalized
investigation of individual statistics in decentralized settings is given
by Wu \& Lall \cite{Wu14CDC}. 
However, these notions are not a sufficient statistic for the past joint policy
from a \emph{designer perspective} (as is the proposed plan-time sufficient statistic in this paper). In fact, they are complementary and we hypothesize that they can be fruitfully combined in future work.

There are works from game theory literature that present structural results
on the value function of so-called ‘repeated zero-sum games with incomplete information’ \cite{mertens1971value,ponssard1975zero,ponssard1980some,ponssard1973zero}.
These can best be understood as a class of two-player extensive form games that lie in between
Bayesian games and POSGs: at the start of the game, nature determines the state
(a joint type) from which each agent makes a private observation
(i.e., individual type), and subsequently the agents take actions in turns, thereby observing the actions of the opponent. 
The models for which these results have been proven are therefore substantially less general than the zs-POSG model we consider.

For various flavors of such games, it has been shown that the value function has a concave/convex
structure: cases with incomplete information on one side \cite{ponssard1973zero,sorin2003stochasticgameswithincompleteinformation}, and cases with
incomplete information on both sides where `observations are independent' (i.e., where the 
distribution over joint types is a product of individual type distributions)
\cite{ponssard1975zero} or dependent (general joint type distributions)
\cite{mertens1971value,ponssard1980some}.
These results, however, crucially depend on the alternating actions and the static state
and therefore do not extend to zs-POSGs.

A game-theoretic model that is closely related to the POSG model is the \emph{Interactive POMDP} or \emph{I-POMDP}
\cite{gmytrasiewicz2004interactive}.  In I-POMDPs, a (subjective) belief $b_i(s, \zeta_j)$ is
constructed from the perspective of a single agent as a probability
distribution over states and the \emph{types}, $ \zeta_j $, of the other agent.
A \emph{level-$k$} I-POMDP agent $i$ reasons about level-$(k-1)$ types $\zeta_j$.
Since each $\zeta_j$ fully determines the future policy of the other agent~$j$,
an I-POMDP can be interpreted as a best-response POMDP similar to the one
introduced by Nair et al.~\cite{nair2003taming} (discussed in
Section~\ref{sec:posg:reductiontoNOSG}), with the difference that the state also
represents which policy the other agent uses. The differences mentioned in 
Section~\ref{sec:posg:reductiontoNOSG}
also apply here; where an I-POMDP can be used to compute a best-response, our
formulation is aimed at computing a Nash equilibrium.

\section{Conclusions and Future Work}
This paper presents a structural result on the shape of the value function of two-player zero-sum games of incomplete information, for games of static state and dynamic state, typically modeled as a Bayesian Game (BG) and Partially Observable Stochastic Game (POSG) respectively.
We formally defined the value function for both types of games in terms of an information distribution called the sufficient plan-time statistic: a probability distribution over joint sets of private information (originally used in the collaborative setting \cite{Oliehoek13IJCAI}). 
Using the fact that this probability distribution can be decomposed into a marginal and a conditional term, 
we presented that in the zero-sum case value functions of both types of games exhibit concavity (convexity) in the space of marginal statistics of the maximizing (minimizing) agent, for every conditional statistic.
In the multi-stage game, this structure of the value function is preserved on every stage. 
Thus, our formulation enables us to make statements about how value generalizes as a function of the information distribution.
Lastly, we showed how the results allow us to reduce our zs-POSG to a model with shared observations, thereby transferring properties of this narrower class of games to the zs-POSG.

We hope that this result leads to solution methods that exploit the structure of the value function at every stage, as recently such developments have been made in the field of cooperative multi-agent problems \cite{MacDermed13NIPS26}.
In particular, we believe that heuristic methods that identify useful (conditional) statistics to explore, or point-based methods that iteratively select statistics to evaluate \cite{MacDermed13NIPS26,spaan2005perseus} may be adapted for the zs-POSG case.

\vspace{.2cm}
\noindent \textbf{Acknowledgments}
\noindent This research is supported by the NWO Innovational Research Incentives Scheme Veni (\#639.021.336) and NWO DTC-NCAP (\#612.001.109) project.

\appendix 

\section{Appendix}
{ % <-- begin appendix

\numberwithin{equation}{section}

\begin{proof}[Proof of Lemma \ref{lemma:VBR_linear_in_statspace_FOBG}]
We will prove item 1: 
\begin{align*}
\V[\text{BR1}]{\fobg}(\sigma_{m,1} \sigma_{c,1}, \delta_2) =  \vec{\sigma}_{m,1} \cdot \vec{r}_{[ \sigma_{c,1}, \delta_2 ]} 
\end{align*}

\noindent
The Q-value definition is expanded in order to bring the marginal term $\statmarg[1]{}$ to the front of the equation:
\begin{IEEEeqnarray}{rCl}
\IEEEeqnarraymulticol{3}{l}{
\Q{\fobg}(\sigma, \delta) 
	\numeq{\ref{eq:bg:qvalue-definition}, \ref{eq:fobg:qvalue-definition}}
		\sum_{\theta \in \Theta} \sigma(\theta) 
			\sum_{a \in \mathcal{A}} 
				\delta(a | \theta) R(\theta, a) 
}\nonumber \\ 
	\, &=&	
        \sum_{\theta_1 \in \Theta_1}  \hspace{-.5mm}
            \sigma_{m,1}(\theta_1) 
		    \sum_{\theta_2 \in \Theta_2}  \hspace{-.5mm}
               \sigma_{c,1}(\theta_2 | \theta_1) 
	    \sum_{a_1 \in \mathcal{A}_1} \hspace{-.5mm}
			 \delta_1(a_1 | \theta_1) 
			 \nonumber \\	&&\>
	    \sum_{a_2 \in \mathcal{A}_2} \hspace{-.5mm}
			 \delta_2(a_2 | \theta_2)   
			R(\jointtype{}, \jointact{} ) . \hspace{5mm}
 \label{eq:fobg:expanded-QR} 
\end{IEEEeqnarray}
A maximization over stochastic decision rules conditioned on $\theta_1$ is equivalent to choosing a maximizing action for each $\theta_1$. Thus, we can rewrite the best-response value function as follows:
\begin{IEEEeqnarray*}{rcl}
\IEEEeqnarraymulticol{3}{l}{
\V[\text{BR1}]{\fobg}(\sigma, \delta_2)  =
		\max\limits_{\delta_1 \in \sdas{1}} 
		\Q[R]{t}(\sigma, \langle \delta_1,  \delta_2 \rangle )  
}\nonumber \\ 
	&\numeq{\ref{eq:fobg:expanded-QR}}&
		  \max\limits_{\delta_1 \in \sdas{1}}
		  \bigg[ \sum_{\theta_1 \in \Theta_1} 
            \sigma_{m,1}(\theta_1) 
		    \sum_{\theta_2 \in \Theta_2} 
               \sigma_{c,1}(\theta_2 | \theta_1) 
	    \sum_{a_1 \in \mathcal{A}_1} \hspace{-1mm}
			 \delta_1(a_1 | \theta_1)  
\nonumber \\ &&\>		
	    \sum_{a_2 \in \mathcal{A}_2} \hspace{-1mm}
			 \delta_2(a_2 | \theta_2)   
			R(\langle \theta_1, \theta_2 \rangle, \langle a_1, a_2 \rangle) \bigg] \\
     &=& \sum_{\theta_1 \in \Theta_1} 
           \sigma_{m,1}(\theta_1) 
           \max\limits_{a_1}  \hspace{-0.5mm}
			\bigg[		   
		    	   \sum_{\theta_2 \in \Theta_2}  \hspace{-1mm}
               \sigma_{c,1}(\theta_2 | \theta_1) 
      	   \hspace{-1mm}
      	    \sum_{a_2 \in \mathcal{A}_2}
      	    \hspace{-1mm}
			 \delta_2(a_2 | \theta_2)  
			R(\theta, a)
           \vphantom{\sum_{a_2}} \bigg]  \\ 
    &\numeq{\ref{eq:fobg:rvector}}&
      	 \sum_{\theta_1 \in \Theta_1} 
          \sigma_{m,1}(\theta_1) 
           \vec{r}_{[\sigma_{c,1}, \delta_2 ]}(\theta_1) 
    	  \numeq{\text{vec. not.}} \vec{\sigma}_{m,1} \cdot
           \vec{r}_{[\sigma_{c,1}, \delta_2 ]} . 
\end{IEEEeqnarray*}
As it is possible to write $\V[\text{BR1}]{\fobg}$ as an inner product of two vectors, $\V[\text{BR1}]{\fobg}$ is linear in $\Delta(\Theta_1)$ for all $\sigma_{c,1}$ and $\delta_2$.
Analogously, $\V[\text{BR2}]{\fobg}$ is linear in $\Delta(\Theta_2)$ for all $\sigma_{c,2}$ and $\delta_1$. 
\end{proof}

\begin{proof}[Proof of Lemma \ref{lemma:sufficiency_statistic}]
\label{appendix:sufficiency_statistic}
The proof is largely identical to the proof of correctness of sufficient statistics in the collaborative setting \cite{Oliehoek13IJCAI}.
 For the final stage $t = h-1$, we have the following:
$
\optQ{t}(\pjp{t}, \AOH{t}, \dr{t})
	= R(\AOH{t}, \dr{t}) = \optQ{t}(\stat{t}, \AOH{t}, \dr{t}).
$
As induction hypothesis, we assume that at stage $t+1$, $\stat{t+1}$ is a sufficient statistic, i.e.: 
\begin{IEEEeqnarray}{rCl}
\optQ{t+1}(\pjp{t+1}, \AOH{t+1}, \dr[\text{pp}]{t+1}) 
	= \optQ{t+1}(\stat{t+1}, \AOH{t+1}, \dr[\text{ss}]{t+1}). \hspace{5mm}
\label{eq:induction_hypothesis_statistics}
\end{IEEEeqnarray}
We aim to show that at stage $t$, $\stat{t}$ is a sufficient statistic:
\begin{IEEEeqnarray}{rCl}
\optQ{t}(\pjp{t}, \AOH{t}, \dr[\text{pp}]{t}) 
	&=& \optQ{t}(\stat{t}, \AOH{t}, \dr[\text{ss}]{t}). \quad
\label{eq:qvalues_equal_induction_proof}
\end{IEEEeqnarray}
We substitute the induction hypothesis into \eqref{eq:posg:qvalue-ss}:
\begin{IEEEeqnarray*}{rcl}
	 \optQ{t}(\pjp{t}, \AOH{t}, \dr[\text{pp}]{t}) &\numeq{\ref{eq:posg:qvalue-ss}} 	&
		 R(\AOH{t}, \dr[\text{pp}]{t}) +  
	  \sum_{\act{t} \in \mathcal{A}} \sum_{\obs{t+1}} 
            \text{Pr}(\AOH{t+1} | \AOH{t}, \dr[\text{pp}]{t}) \nonumber \\ && 
            \optQ{t+1}(\pjp{t+1}, \AOH{t+1}, \optdr[\text{pp}]{t+1}) \\
	&\numeq{\ref{eq:induction_hypothesis_statistics}} &
	    R(\AOH{t}, \dr[\text{ss}]{t}) 
  		+
        \sum_{\act{t} \in \mathcal{A}} \sum_{\obs{t+1}} 
			\text{Pr}(\AOH{t+1} | \AOH{t}, \dr[\text{ss}]{t}) \nonumber \\ && 
            \optQ{t+1}(\stat{t+1}, \AOH{t+1}, \optdr[\text{ss}]{t+1}) 
=
   		\optQ{t}(\stat{t}, \AOH{t}, \dr[\text{ss}]{t}).
\end{IEEEeqnarray*}

\noindent
Furthermore, decision rules $\optdr[1,pp]{t+1}$ (based on the past joint policy) and $\optdr[1,ss]{t+1}$ (based on the sufficient statistic) are equal:
\begin{IEEEeqnarray*}{rcl}
\optdr[1,pp]{t+1}
	&\numeq{\eqref{eq:posg:decision-selection-value-ss1}}&
		 \argmax\limits_{\dr[1]{t+1} \in \sdas{1}}              
		 	\min\limits_{\dr[2]{t+1} \in \sdas{2}} \bigg[ \vphantom{\sum_\theta} 
		 	\sum_{\AOH{t+1} \in \AOHs{t+1}} 
           \text{Pr}(\AOH{t+1} | b^0, \pjp{t+1}) \nonumber \\ && 
           \optQ{t+1}(\pjp{t+1}, \AOH{t+1},
		    \jointdr{t+1}
            )  
             \vphantom{\sum_\theta} \bigg]
             \\
         &\numeq{\eqref{def:plantimesufficientstatistic}} &
             \argmax\limits_{\dr[1]{t+1} \in \sdas{1}}  
             \min\limits_{\dr[2]{t+1} \in \sdas{2}} 
             \bigg[
             \sum_{\AOH{t+1} \in \AOHs{t+1}} 
             \stat{t+1}(\AOH{t+1}) \nonumber \\ && 
             \optQ{t+1}(\stat{t+1}, \AOH{t+1}, \jointdr{t+1} )
			 \bigg] \numeq{\ref{eq:posg:decision-selection-value-ss1}}
             \optdr[1,ss]{t+1} .
\end{IEEEeqnarray*}

\noindent
Analogous reasoning holds for $\optdr[2,ss]{t+1}$. Thus, by induction, $\stat{t}$ is a sufficient statistic for $\pjp{t}$, $\forall t \in 0 \ldots h-1$.
\end{proof}

\begin{proof}[Proof of Lemma \ref{lemma:VBR_linear_in_statspace}]
\label{appendix:vbr_linear_in_statspace}
By the results of Lemma \ref{lemma:VBR_linear_in_statspace_FOBG} and the results from Section \ref{sec:theoryPOSG:equivalence}, we know the best-response value function $\V[\text{BR1}]{h-1}$ to be linear in $\statspace[1]{h-1}$.
For all other stages, we assume the following induction hypothesis:
\begin{IEEEeqnarray}{rCl}
\V[\text{BR1}]{t+1}(\statmarg[1]{t+1}\statcond[1]{t+1}, \ppol[2]{t+1} )   
	=  \statmarg[1]{t+1} \cdot
          \vec{\nu}_{[ \statcond[1]{t+1}, \ppol[2]{t+1}]}.
\label{eq:inductionhypothesis}
\end{IEEEeqnarray}
\noindent
For the inductive step we aim to prove that at the current stage $t$ the following holds:
\begin{IEEEeqnarray}{rCl}
\V[\text{BR1}]{t}(\statmarg[1]{t} \statcond[1]{t}, \ppol[2]{t})   
	=  \vstatmarg[1]{t} \cdot
          \vec{\nu}_{[ \statcond[1]{t}, \ppol[2]{t}]}.
\label{eq:maintheorem_to_proof}
\end{IEEEeqnarray}
\noindent
Let $\Q[R]{t}$ be a function similar to \eqref{eq:fobg:qvalue-definition} that defines reward for a given statistic and joint decision rule. 
We expand the definition of $\V[\text{BR1}]{t}$.
For notational convenience we write $\stat{t}$ instead of $\statmarg[i]{t} \statcond[i]{t}$, but keep in mind that we consider statistics corresponding to conditional $\statcond[i]{t}$. 
\begin{IEEEeqnarray}{rcl}
\IEEEeqnarraymulticol{3}{l}{
	\V[\text{BR1}]{t}(
		\stat{t}, \ppol[2]{t}	
	) 
	\numeq{\ref{eq:posg:bestresponsedef:derivation}}
		\max\limits_{\ppol[1]{t} \in \Pi_1^t} 
		\V{t}(
			\stat{t}, \jointppol{t}
		) 
} \nonumber \\  
&=& \hspace{-1mm}
		\max\limits_{\ppol[1]{t} \in \Pi_1^t} 
		\bigg[
	    \Q[R]{t}(
			\stat{t}, \jointdr{t}			
	    	) + 			
			\V{t+1}( \text{U}_\text{ss}( \stat{t}, \dr{t} ), \jointppol{t+1})		
	    \bigg] \nonumber \\
	&=& \hspace{-1mm} 
		\max\limits_{\dr[1]{t} \in \sdas{1}} 
		\max\limits_{\ppol[1]{t+1}} 
		\bigg[
	    \Q[R]{t}(
			\stat{t+1}, \jointdr{t}
	    	) 
	    +  
	    \V{t+1}(
			\stat{t+1}, \jointppol{t+1}
		) 
	    \bigg] \nonumber \\
	&=& \hspace{-1mm}
		\max\limits_{\dr[1]{t} \in \sdas{1}} \bigg[ 
			\vphantom{\max\limits_{\ppol[1]{t+1}}}
	    		\Q[R]{t}(
				\stat{t}, \jointdr{t}
	    		) + 
	    \max\limits_{\ppol[1]{t+1}} \bigg[
	   		\V{t+1}(
				\stat{t+1}, \jointppol{t+1}	
   			) 
	    \bigg] \bigg] \nonumber  \\	    
\hspace{-.5mm}	&\numeq{\ref{eq:posg:bestresponsedef:derivation}}& \hspace{-.5mm}
		\max\limits_{\dr[1]{t} \in \sdas{1}}
	    \bigg[ 
	    		\Q[R]{t}(
				\stat{t}, \jointdr{t}
	    		) + 
	    		\V[\text{BR1}]{t+1}(
	    			\stat{t+1}, \jointppol{t+1}
	    		)
	    \bigg] .
\label{eq:posg:immediate-result-brvaluefunction}
\end{IEEEeqnarray}
Here, $\text{U}_\text{ss}$ is the statistic update rule, defined in accordance with \eqref{eq:posg:statistic-update-rule}.
We make the decomposition of $\stat{t}$ into the marginal and conditional terms explicit again.
Immediate reward $\Q[R]{t}$ can be expanded similar to \eqref{eq:fobg:expanded-QR}:
\begin{IEEEeqnarray}{rCl}
\IEEEeqnarraymulticol{3}{l}{
\Q[R]{t}( \statmarg[1]{t} \statcond[1]{t}, \jointdr{t} ) =
 		\statmarg[1]{t}(\AOH[1]{t})
    		\statcond[1]{t}(\AOH[2]{t} | \AOH[1]{t})
    	   	\sum\limits_{\act[1]{t} \in \mathcal{A}_1}
    	   		\dr[1]{t}(\act[1]{t} | \AOH[1]{t} )
}  \nonumber  \\ \, 	
\quad\quad 	   	\sum\limits_{\act[2]{t} \in \mathcal{A}_2}
    			\dr[2]{t}(\act[2]{t} | \AOH[2]{t} )
		R(\jointAOH{t}, \jointact{t}) .
\label{eq:posg:expanded-QR}
\end{IEEEeqnarray}		
\noindent
We expand $\V[\text{BR1}]{t+1}$ using the induction hypothesis in order to bring the marginal distribution $\statmarg[1]{t}$ to the front:
\begin{IEEEeqnarray}{rCl}
\IEEEeqnarraymulticol{3}{l}{
	\V[\text{BR1}]{t+1}(
			\statmarg[1]{t+1}\statcond[1]{t+1}, \ppol[2]{t+1}   
		) 
	\numeq{\ref{eq:inductionhypothesis}} 
		\statmarg[1]{t+1} \cdot
        \vec{\nu}_{[ \statcond[1]{t+1}, \ppol[2]{t+1}]} 
     } \nonumber \\ 
        &=&	\sum\limits_{\AOH[1]{t+1} \in \AOHs[1]{t+1}}
  			\statmarg[1]{t+1}(\AOH[1]{t+1})
    		\vec{\nu}_{[ \statcond[1]{t+1}, \ppol[2]{t+1}]}(\AOH[1]{t+1}) \nonumber \\ &\numeq{\ref{eq:posg:statistic-update-rule}}&
   		\sum\limits_{\AOH[1]{t} \in \AOHs[1]{t}}
	 		\statmarg[1]{t}(\AOH[1]{t})
    		\sum\limits_{\AOH[1]{t} \in \AOHs[2]{t}}
    	   		\statcond[1]{t}(\AOH[2]{t} | \AOH[1]{t})
    	   	\sum\limits_{\act[1]{t} \in \mathcal{A}_1}
    	   		\dr[1]{t}(\act[1]{t} | \AOH[1]{t} )
\nonumber \\ && 
    	   	\sum\limits_{\act[2]{t} \in \mathcal{A}_2}
    			\dr[2]{t}(\act[2]{t} | \AOH[2]{t} ) 
    		\sum\limits_{\obs{t+1} \in \mathcal{O}_1} \sum\limits_{\obs{t+1} \in \mathcal{O}_2}
    		\nonumber \\ && 
    		\text{Pr}( \jointobs{t+1} |
    				   \jointAOH{t},
   				   \jointact{t})
    			\vec{\nu}_{[ \statcond[1]{t+1}, \ppol[2]{t+1}]}
    			(\AOH[1]{t+1}). \hspace{9mm}
\label{eq:expanded_VBR_tplus1} 
\end{IEEEeqnarray}
Filling the expanded equations into \eqref{eq:posg:immediate-result-brvaluefunction} and factorizing gives:
\begin{IEEEeqnarray}{rcl}
\IEEEeqnarraymulticol{3}{l}{
	\V[\text{BR1}]{t}(
		\statmarg[1]{t}\statcond[1]{t},  \ppol[2]{t}
		) 	
	\numeq{\ref{eq:posg:bestresponsedef:derivation}}
\max\limits_{\dr[1]{t} \in \sdas{1}}
	    \bigg[ 
	    		\Q[R]{t}(\statmarg{t}, \jointdr{t}  | \statcond[1]{t}) + 
      } \nonumber \\
&&	    		\V[\text{BR1}]{t+1}(\statmarg[1]{t+1}, 
	    		                    \ppol[2]{t+1} | \statcond[1]{t+1})
	    \bigg] 
	 \numeq{\ref{eq:posg:expanded-QR}, 
	 		 \ref{eq:expanded_VBR_tplus1}} \nonumber \\
&&     \max\limits_{\dr[1]{t} \in \sdas{1}}
	    \bigg[ 
		    	\sum_{\AOH{t} \in \AOHs[1]{t}} 
               \statmarg[1]{t}(\AOH[1]{t}) 
		    \sum_{\AOH{t} \in \AOHs[2]{t}} 
                \statcond[1]{t}(\AOH[2]{t} | \AOH[1]{t})  
     \sum\limits_{\act[1]{t} \in \mathcal{A}_1}
\nonumber \\ &&\>
			   \dr[1]{t}(\act[1]{t} | \AOH[1]{t})
	    	    \sum\limits_{\act[1]{t} \in \mathcal{A}_2}
			   \dr[2]{t}(\act[2]{t} | \AOH[2]{t}) 
		\bigg( 
		 R(\jointAOH{t},  \jointact{t}) + 
\nonumber \\ &&\>		 
		 		 \sum\limits_{\obs{t+1} \in \mathcal{O}}
    			\text{Pr}( \obs{t+1} | \AOH{t}, \act{t} )
    			\vec{\nu}_{[ \statcond[1]{t+1}, \ppol[2]{t+1}]} (\AOH[1]{t+1}) 
    	    \bigg) \bigg]. 
\label{eq:result_factorization}
\end{IEEEeqnarray}
Note that the vector is indexed by the conditional $\statcond[1]{t+1}$.
While this conditional is dependent on $\dr[2]{t}$, it is not dependent on $\dr[1]{t}$, allowing us to remove the maximization over decision rules $\dr[1]{t}$ from the equation.
As a maximization over decision rules conditioned on individual AOH $\AOH[1]{t}$ is equal to choosing the maximizing action for each of these AOHs, we can rewrite \eqref{eq:result_factorization} as follows: 
\begin{IEEEeqnarray}{rcl}
\IEEEeqnarraymulticol{3}{l}{
	\V[\text{BR1}]{t}(\stat{t}, \ppol[2]{t})
} \nonumber \\
    =
& 
    \sum\limits_{\AOH[1]{t} \in \AOHs[1]{t}} \statmarg[1]{t}(\AOH[1]{t}) \max\limits_{\act[1]{t}} 
            \bigg[ \sum_{\AOH{t} \in \AOHs[2]{t}} \statcond[1]{t}(\AOH[2]{t} |
                \AOH[1]{t}) \sum\limits_{\act[1]{t} \in \mathcal{A}_2}
                \dr[2]{t}(\act[2]{t} | \AOH[2]{t}) 
\nonumber 
\\ 
&
        \bigg( 
        \>R(\AOH{t}, \act{t} ) 
	+
        \sum\limits_{\obs{Ft+1} \in \mathcal{O}} \text{Pr}( \obs{t+1} | \AOH{t}, \act{t}) 
\vec{\nu}_{[ \statcond[1]{t+1}, \ppol[2]{t+1}]} (\AOH[1]{t+1}) \bigg) \bigg]
\nonumber \\
        \numeq{\ref{eq:concaveconvex:valuevector}}
&
    		\sum\limits_{\AOH[1]{t} \in \AOHs[1]{t}}
	    			\statmarg[1]{t}(\AOH[1]{t})
	    	\vec{\nu}_{[ \statcond[1]{t}, \ppol[2]{t} ]}(\AOH[1]{t})	 
	\numeq{\text{vec. not.}} 	\vstatmarg[1]{t} \cdot
	    	\vec{\nu}_{[ \statcond[1]{t}, \ppol[2]{t} ]}. \hspace{10mm}    \label{eq:theorem_linearity_result}
\end{IEEEeqnarray}
 This corresponds to \eqref{eq:maintheorem_to_proof}.
Therefore, by induction, best-response value function 
$\V[\text{BR1}]{t}$
is linear in $\statspace[1]{t}$ for a given $\statcond[1]{t}$ and $\ppol[2]{t}$, for all stages $t=0\ldots h-1$.
Analogously, $\V[\text{BR2}]{t}$ is a linear function in $\statspace[2]{t}$ for a given $\statcond[2]{t}$ and $\ppol[1]{t}$, for all stages $t=0\ldots h-1$.
\end{proof}

} %<-- end appendix

\newpage

\bibliographystyle{ecai}
\bibliography{bibliography}

\providecommand{\noopsort}[1]{}
\begin{thebibliography}{10}

\bibitem{aubin1998optima}
Jean-Pierre Aubin, {\em Optima and equilibria: an introduction to nonlinear
  analysis}, volume 140, Springer Science \& Business Media, 1998.

\bibitem{Boyd04book}
Stephen Boyd and Lieven Vandenberghe, {\em Convex Optimization}, Cambridge
  University Press, 2004.

\bibitem{Dibangoye13IJCAI}
Jilles~S. Dibangoye, Christopher Amato, Olivier Buffet, and Fran\c{c}ois
  Charpillet, `Optimally solving {Dec-POMDPs} as continuous-state {MDPs}', in
  {\em Proceedings of the International Joint Conference on Artificial
  Intelligence}, (2013).

\bibitem{emery2005approximate}
R.~Emery-Montemerlo, G.~Gordon, J.~Schneider, and S.~Thrun, `Approximate
  solutions for partially observable stochastic games with common payoffs', in
  {\em Proceedings of International Joint Conference on Autonomous Agents and
  Multi Agent Systems}, (2004).

\bibitem{ghosh2004zero}
MK~Ghosh, D~McDonald, and S~Sinha, `Zero-sum stochastic games with partial
  information', {\em Journal of Optimization Theory and Applications}, {\bf
  121}(1),  99--118, (2004).

\bibitem{gmytrasiewicz2004interactive}
Piotr~J Gmytrasiewicz and Prashant Doshi, `Interactive {POMDPs}: {P}roperties
  and preliminary results', in {\em Proceedings of the Third International
  Joint Conference on Autonomous Agents and Multiagent Systems-Volume 3}, pp.
  1374--1375. IEEE Computer Society, (2004).

\bibitem{hansen2004dynamic}
Eric~A Hansen, Daniel~S. Bernstein, and Shlomo Zilberstein, `Dynamic
  {Programming} for {Partially Observable Stochastic Games}', in {\em
  Proceedings of the National Conference on Artificial Intelligence}, volume~4,
  pp. 709--715, (2004).

\bibitem{jain2011double}
Manish Jain, Dmytro Korzhyk, Ond{\v{r}}ej Van{\v{e}}k, Vincent Conitzer, Michal
  P{\v{e}}chou{\v{c}}ek, and Milind Tambe, `A double oracle algorithm for
  zero-sum security games on graphs', in {\em Proceedings of the International
  Joint Conference on Autonomous Agents and Multi Agent Systems}, pp. 327--334,
  (2011).

\bibitem{koller1994fast}
Daphne Koller, Nimrod Megiddo, and Bernhard Von~Stengel, `Fast algorithms for
  finding randomized strategies in game trees', in {\em Proceedings of the
  twenty-sixth annual ACM symposium on Theory of computing}, pp. 750--759. ACM,
  (1994).

\bibitem{MacDermed13NIPS26}
Liam~C. MacDermed and Charles Isbell, `Point based value iteration with optimal
  belief compression for {Dec-POMDPs}', in {\em Advances in Neural Information
  Processing Systems 26}, pp. 100--108, (2013).

\bibitem{mertens1971value}
Jean-Francois Mertens and Shmuel Zamir, `The value of two-person zero-sum
  repeated games with lack of information on both sides', {\em International
  Journal of Game Theory}, {\bf 1}(1),  39--64, (1971).

\bibitem{nair2003taming}
Ranjit Nair, Milind Tambe, Makoto Yokoo, David Pynadath, and Stacy Marsella,
  `Taming {D}ecentralized {POMDPs}: {T}owards efficient policy computation for
  multiagent settings', in {\em Proceedings of the International Joint
  Conference on Artificial Intelligence}, pp. 705--711, (2003).

\bibitem{nayyar2014common}
Ashutosh Nayyar, Abhishek Gupta, Cedric Langbort, and Tamer Basar, `Common
  information based {M}arkov perfect equilibria for {S}tochastic {G}ames with
  asymmetric information: {F}inite {G}ames', {\em Automatic Control, IEEE
  Transactions on}, {\bf 59}(3),  555--570, (2014).

\bibitem{Nayyar13TAC}
Ashutosh Nayyar, Aditya Mahajan, and Demosthenis Teneketzis, `Decentralized
  stochastic control with partial history sharing: A common information
  approach', {\em IEEE Transactions on Automatic Control}, {\bf 58},
  1644--1658, (July 2013).

\bibitem{Oliehoek13IJCAI}
Frans~A. Oliehoek, `Sufficient {P}lan-{T}ime {S}tatistics for {D}ecentralized
  {POMDPs}', in {\em Proceedings of the Twenty-Third International Joint
  Conference on Artificial Intelligence}, pp. 302--308, (2013).

\bibitem{Oliehoek14IASTR}
Frans~A. Oliehoek and Christopher Amato, `{Dec-POMDPs} as non-observable
  {MDPs}', IAS technical report IAS-UVA-14-01, Intelligent Systems Lab,
  University of Amsterdam, Amsterdam, The Netherlands, (October 2014).

\bibitem{oliehoek2008optimal}
Frans~A Oliehoek, Matthijs~TJ Spaan, and Nikos~A Vlassis, `Optimal and
  approximate q-value functions for {Decentralized POMDPs}.', {\em Journal of
  AI Research}, {\bf 32},  289--353, (2008).

\bibitem{oliehoek2006dec}
Frans~A. Oliehoek and Nikos Vlassis, `{Dec-POMDPs} and extensive form games:
  equivalence of models and algorithms', {\em Ias technical report
  IAS-UVA-06-02, University of Amsterdam, Intelligent Systems Lab, Amsterdam,
  The Netherlands}, (2006).

\bibitem{Osborne+Rubinstein94}
Martin~J. Osborne and Ariel Rubinstein, {\em A Course in Game Theory}, chapter
  Nash Equilibrium,  21 -- 27, The MIT Press, July 1994.

\bibitem{pineau2006anytime}
Joelle Pineau, Geoffrey~J Gordon, and Sebastian Thrun, `Anytime point-based
  approximations for large {POMDPs}.', {\em Journal of AI Research}, {\bf 27},
  335--380, (2006).

\bibitem{ponssard1975zero}
Jean-Pierre Ponssard, `Zero-sum games with "almost" perfect information', {\em
  Management Science}, {\bf 21}(7),  794--805, (1975).

\bibitem{ponssard1980some}
Jean-Pierre Ponssard and Sylvain Sorin, `Some results on zero-sum games with
  incomplete information: The dependent case', {\em International Journal of
  Game Theory}, {\bf 9}(4),  233--245, (1980).

\bibitem{ponssard1973zero}
Jean-Pierre Ponssard and Shmuel Zamir, `Zero-sum sequential games with
  incomplete information', {\em International Journal of Game Theory}, {\bf
  2}(1),  99--107, (1973).

\bibitem{rubin2011computer}
Jonathan Rubin and Ian Watson, `Computer poker: A review', {\em Artificial
  Intelligence}, {\bf 175}(5),  958--987, (2011).

\bibitem{Shoham08book}
Yoav Shoham and Kevin Leyton-Brown, {\em Multiagent systems: Algorithmic,
  game-theoretic, and logical foundations}, Cambridge University Press, 2008.

\bibitem{sorin2003stochasticgameswithincompleteinformation}
Sylvain Sorin, `{Stochastic Games} with incomplete information', in {\em
  Stochastic Games and applications}, eds., Abraham Neyman and Sylvain Sorin,
  volume 570, Springer, (2003).

\bibitem{spaan2005perseus}
Matthijs~T.J. Spaan and Nikos Vlassis, `Perseus: {Randomized} point-based value
  iteration for {POMDPs}.', {\em Journal of AI Research}, {\bf 24},  195--220,
  (2005).

\bibitem{Wu14CDC}
Jeffrey Wu and Sanjay Lall, `A theory of sufficient statistics for teams', in
  {\em Proc. of the 53rd Annual Conference on Decision and Control}, pp.
  2628--2635. IEEE, (2014).

\end{thebibliography}

\end{document}